\documentclass[a4paper,11pt]{article}

\usepackage[latin1]{inputenc}
\usepackage[english]{babel}
\usepackage{graphicx}
\usepackage{amssymb}
\usepackage{amsmath}%
\usepackage{amsfonts}%
\usepackage{amsthm}
\usepackage{url}

\newtheorem{theorem}{Theorem}

\newtheorem{corollary}{Corollary}

\newtheorem{conjecture}{Conjecture}

\newtheorem{example}{Example}
\newtheorem{lemma}{Lemma}

\newtheorem{remark}{Remark}
\newtheorem{definition}{Definition}
\numberwithin{equation}{section}
\DeclareMathOperator{\argmin}{argmin}

\def \matlab    {MATLAB$^{\text{\tiny \textregistered}}$}
\usepackage{algorithm}
\usepackage{algorithmic}

\DeclareMathOperator{\conv}{conv} 
\DeclareMathOperator{\cone}{cone} 
\DeclareMathOperator{\col}{col} 
\DeclareMathOperator{\rank}{rank}

\title{Sparse and Unique Nonnegative Matrix Factorization \\ 
Through Data Preprocessing}
\author{Nicolas Gillis \\ University of Waterloo, Department of Combinatorics and Optimization, \\ Waterloo, Ontario N2L 3G1, Canada. \\ E-mail: ngillis@uwaterloo.ca} 

\date{}

\begin{document}
\renewcommand{\labelitemi}{$\diamond$}

\maketitle

\begin{abstract}
Nonnegative matrix factorization (NMF) has become a very popular technique in machine learning because it automatically extracts meaningful features through a sparse and part-based representation. However, NMF has the drawback of being highly ill-posed, that is, there typically exist many different but equivalent factorizations. In this paper, we introduce a completely new way to obtaining more well-posed NMF problems whose solutions are sparser. Our technique is based on the preprocessing of the nonnegative input data matrix, and relies on the theory of M-matrices and the geometric interpretation of NMF. This approach provably leads to optimal and sparse solutions under the separability assumption of Donoho and Stodden \cite{DS03}, and, for rank-three matrices,  makes the number of exact factorizations finite.  We illustrate the effectiveness of our technique on several image datasets.  
\end{abstract} 


\section{Introduction} \label{intro}

Given an $m$-by-$n$ nonnegative matrix $M \geq 0$ and a factorization rank $r$, nonnegative matrix factorization (NMF) looks for two nonnegative matrices $U$ and $V$ of dimension $m$-by-$r$ and $r$-by-$n$ respectively such that $M \approx UV$. To assess the quality of an approximation, 
a popular choice is the Frobenius norm of the residual $||M-UV||_F$  and NMF can for example be formulated as the following optimization problem 
\begin{equation} \label{nmf} 
\min_{U \in \mathbb{R}^{m \times r}, V \in \mathbb{R}^{r \times n}} \; ||M-UV||_F^2 \quad \text{ such that } U \geq 0 \text{ and } V \geq 0. 
\end{equation} 
Assuming that $M$ is a matrix where each column represents an element of a dataset (e.g., a vectorized image of pixel intensities), NMF can be interpreted in the following way. Since $M_{:j} \approx \sum_{k = 1}^r U_{:k} V_{kj}$ $\forall j$, each column $M_{:j}$ of $M$ is reconstructed using an additive linear combination of nonnegative basis elements (the columns of $U$).  These basis elements can  be interpreted in the same way as the columns of $M$ (e.g., as images). Moreover, they can only be summed up (since $V$ is nonnegative) in order to approximate the original data matrix $M$ which leads to a part-based representation: NMF will automatically extract localized and meaningful features from the dataset. The most famous illustration of such a decomposition is when the columns of $M$ represent facial images for which NMF is able to extract common features such as eyes, nodes and lips \cite{LS99}; see Figure~\ref{cbclB} in Section~\ref{appl}.  

NMF 
has become a very popular data analysis technique 
and has been successfully used in many different areas, e.g., hyperspectral imaging \cite{PPP06}, text mining \cite{XLG03}, clustering \cite{DHS05}, air emission control \cite{PT94}, blind source separation \cite{CAZP09}, and music analysis \cite{FBD09}. 

\subsection{Geometric Interpretation of NMF} 


A very useful tool for understanding NMF better is its geometric interpretation. In fact,  NMF is closely related to a problem in computational geometry consisting in finding a polytope nested between two given polytopes.  In this section, we briefly recall this connection, which will be extensively used throughout the paper. \\

Let $(U,V)$ be an exact NMF
of $M$ (i.e., $M = UV$, $U \geq 0$ and $V \geq 0$), and let us assume that no column of $U$ or ${M}$ is all zeros; otherwise they can be removed without loss of generality. 
\begin{definition}[Pullback map] \label{pbm} 
Given an $m$-by-$n$ nonnegative matrix $X$ without all-zero column, 
$D(X)$ is the $n$-by-$n$ diagonal matrix whose diagonal elements are the inverse of the $\ell_1$-norms of the columns of $X$: 
\begin{equation} \label{diago} 
D(X)_{ii} = ||X_{:i}||_1^{-1} = \left(\sum_{k=1}^m |X_{ki}|\right)^{-1} \; \forall i, \quad D(X)_{ij} = 0 \; \forall \, i \neq j, 
\end{equation}
and $\theta(X) = XD(X)$ is the pullback map of $X$ so that $\theta(X)$ is column stochastic, i.e., $\theta(X)$ is nonnegative and its columns sum to one. 
\end{definition} 
We have that (see, e.g., \cite{CL08}) 
\[
{M} = UV  \iff  \theta(M) = {M} D({M})  = \underbrace{U D(U)}_{\theta(U)} \underbrace{D(U)^{-1} V D({M})}_{V'}  \iff  \theta(M) = \theta(U) V',   
\]
where $V'$ must be column stochastic since $\theta(M)$ and $\theta(U)$ are both column stochastic and $\theta(M)=\theta(U)V'$. Therefore, the columns of $\theta(M)$ are convex combinations (i.e., linear combinations with nonnegative weights summing to one) of the columns of $\theta(U)$. This implies that 
\begin{equation} \label{poly}
\conv(\theta(M)) \quad \subseteq \quad \conv(\theta(U)) \quad \subseteq \quad \Delta^{m}, 
\end{equation}
where $\conv(X)$ denotes the convex hull of the columns of matrix $X$, and $\Delta^{m} = \{  x \in \mathbb{R}^m  \ | \  \sum_i^m x_i = 1, x_i \geq 0 \ 1 \leq i \leq m   \}$ is the unit simplex (of dimension $m-1$). 
An exact NMF ${M} = UV$ can then be geometrically interpreted as 
a polytope $T = \conv(\theta(U))$ nested between an inner polytope $\conv(\theta(M))$ and an outer polytope $\Delta^{m}$. 
\begin{quote}
Hence finding the minimal number of nonnegative rank-one factors to reconstruct $M$ exactly is equivalent to finding a polytope $T$ with minimum number of vertices nested between two given polytopes: the inner polytope $\conv(\theta(M))$ and the outer polytope $\Delta^{m}$. 
\end{quote}
This problem is referred to as the nested polytopes problem (NPP), and is then equivalent to computing an exact nonnegative matrix factorization  \cite{H84} (see also \cite{GG10b} and the references therein).  In the remaining of the paper, we will denote NPP($M$) the NPP instance corresponding to $M$ with inner polytope $\conv(\theta(M))$ and outer polytope $\Delta^{m}$. 


\begin{remark}
The geometric interpretation can also be equivalently characterized in terms of cones, see~\cite{DS03}, for which we have 
\[
\cone(M) \quad \subseteq \quad \cone(U) \quad \subseteq \quad \mathbb{R}_+^m, 
\] 
where $\cone(X) = \{ x | x = Xa, a \geq 0 \}$. 
The geometric interpretation based on convex hulls from Equation~\eqref{poly}   amounts to the intersection of the cones with the hyperplane $\{x|\sum x_i = 1 \}$ (this is the reason why zero columns of $M$ and $U$ need to be discarded in that case). 
\end{remark}

\subsection{Uniqueness of NMF} \label{uniqueNMF}

There are several difficulties in using NMF in practice. In particular, the optimization problem \eqref{nmf} is NP-hard \cite{V09}, and typically only convergence to stationary points is guaranteed by standard algorithms. There does not seem to be an easy way to go around this (except if the factorization rank is very small \cite{AGKM11}) since NMF problems typically have  many local minima. 

Another difficulty is the non-uniqueness: even if one is given an optimal (or  good) NMF $(U,V)$ of $M$, there might exist many equivalent solutions $(UQ,Q^{-1}V)$ for non-monomial\footnote{A monomial matrix is a permutation of a diagonal matrix with positive diagonal elements.} matrices $Q$ with $UQ \geq 0$ and $Q^{-1}V \geq 0$, see, e.g., \cite{LC08}. Such transformations lead to different interpretations, especially when the supports of $U$ and $V$ change. For example, in document classification, each entry $M_{ij}$ of matrix $M$ indicates the `importance' of word $i$ in document $j$ (e.g., the number of appearances of word $i$ in text $j$). The factors $(U,V)$ of NMF are interpreted as follows: the  columns of $U$ represent the topics (i.e., bags of words) while the columns of $V$ link the documents to these topics. The sparsity patterns of $U$ and $V$ are then a crucial characteristic since they indicate which words belong to which topics and which topics is discussed by which documents.

Different approaches exist to obtain (more) well-posed NMF problems and most of them are based on the incorporation of additional constraints into the NMF model, e.g., 
\begin{itemize} 

\item \textbf{Sparsity}. Require the factors in NMF to be sparse. Under some appropriate assumptions, this leads to a unique solution \cite{TST05}. 
Geometrically, requiring the matrix $U$ to be sparse is equivalent to requiring the vertices of the nested polytope $\conv(\theta(U))$ to be located on the low-dimensional faces of the outer polytope $\Delta^{m}$, hence making the problem more well posed. 
In practice, the most popular technique to obtain sparser solutions is to add sparsity inducing penalty terms, such as a $\ell_1$-norm penalty \cite{KP07} (see also Section~\ref{appl}). Another possibility is to use a projection onto the set of sparse matrices \cite{Hoy}. 

\item \textbf{Minimum Volume}. Require the polytope $\conv(\theta(U))$ to have minimum volume, see, e.g., \cite{MQ07, HM10, ZXY11}; this has a long history in hyperspectral imaging \cite{C94}. Again, this constraint is typically enforced using a proper penalty term in the objective function. 
Volume maximization of the polytope is also possible, leading to a sparser factor $U$ (since the columns of $U$ will be encouraged to be on the faces of $\Delta^m$), see, e.g., \cite{WCCW10}, which is essentially equivalent to performing volume minimization for the matrix transpose. In fact, taking the polar of the three polytopes in Equation~\eqref{poly} interchanges the role of the inner and outer polytopes,  while the polar of $\conv(\theta(M))$ is given by $\conv(\theta(M^T))$, see, e.g., \cite[Section 3.6]{G11}.  

\item \textbf{Orthogonality}. Require the columns of matrix $U$ to be orthogonal \cite{DLPP06}. Geometrically, it amounts to position the vertices of $\conv(\theta(U))$ on the low-dimensional faces of $\Delta^{m}$ 
so that if one of the columns of $\theta(U)$ is not on a facet of $\Delta^{m}$ (i.e., $U_{ik} > 0$ for some $i, k$), then all the other columns of $U$ must be on that facet (i.e., $U_{ip} = 0$ $\forall p \neq k$). 
This condition is rather restrictive, but proved successful in some situations, e.g., for clustering, see \cite{DHS05, PGAG11}. 

\end{itemize}


\subsection{Outline of the Paper}

 
In this paper, we address the problem of uniqueness and  introduce a completely new approach to make NMF problems more well posed, and obtain sparser solutions. 
Our technique is based on a preprocessing of the input matrix $M$ to make it sparser while preserving its nonnegativity and its column space. The motivation is based on the geometric interpretation of NMF which shows that \emph{sparser matrices will correspond to more well-posed NMF problems whose solutions are sparser}.

 In Section~\ref{nonu}, we recall how sparsity of $M$ makes the corresponding NMF problem more well posed. In particular, we give a new result linking the support of $M$ and the uniqueness of the corresponding NMF problem. 
   
 In Section~\ref{prep}, we introduce a preprocessing $\mathcal{P}(M) = MQ$ of $M$ where $Q$ is an inverse-positive matrix, i.e., $Q$ has full rank and its inverse $Q^{-1}$ is nonnegative. Hence, if $(U,V')$ is an NMF  of $\mathcal{P}(M)$ with   $\mathcal{P}(M) \approx UV'$, then $(U,V'Q^{-1})$ is an NMF of $M$ since $M = \mathcal{P}(M)Q^{-1} \approx UV'Q^{-1}$ and $V'Q^{-1} \geq 0$.

In Section~\ref{propp}, we prove some important properties of the preprocessing; in particular that it is well-defined, invariant to permutation and scaling, and optimal  under the separability assumption of Donoho and Stodden \cite{DS03}. Moreover, in the exact case for rank-three matrices (i.e., $M = UV$ and $\rank(M)= 3$) we show how the preprocessing can be used to obtain an equivalent NMF problem with a finite number of solutions. 

In Section~\ref{comp2}, we address some practical issues of using the preprocessing: the computational cost, the rescaling of the columns $\mathcal{P}(M)$ and the ability to dealing with sparse and noisy matrices. 

In Section~\ref{appl}, we present some very promising numerical experiments on facial and hyperspectral image datasets.  


\section{Non-Uniqueness, Geometry and Sparsity} \label{nonu}

Let $M \in \mathbb{R}^{m \times n}_+$ and $(U,V) \in \mathbb{R}^{m \times r}_+ \times \mathbb{R}^{r \times n}_+$ be an exact nonnegative matrix factorization of $M$, i.e., $M = UV$. The minimum $r$ such that such a decomposition exists is the nonnegative rank of $M$ and will be denoted $\rank_+(M)$. 
If $U$ is not full rank (i.e., $\rank(U) < r$), then the decomposition is typically not unique. In fact, the convex combinations (i.e., $V \geq 0$) cannot in general  be uniquely determined:
  the  polytope $T = \conv(\theta(U))$ has $r$ vertices while its dimension is strictly smaller than $r-1$ implying that any point in the interior of $T$ can be reconstructed with infinitely many convex combinations of the $r$ vertices of $T$. However, if all columns of $\conv(\theta(M))$ are located on $k$-dimensional faces of $T$ having exactly $k+1$ vertices, then the convex combinations given by $V$ are unique \cite{SX11}. 
 
In practice, it is therefore often implicitly assumed that $\rank_+(M) = \rank(M) = r$  hence $\rank(U) = r$ (since $U$ has $r$ columns and spans the column space of $M$ of dimension $r$);  see the discussion in \cite{AGKM11} and the references therein. 
In this situation, the uniqueness can be characterized as follows: 


\begin{theorem}[\cite{LC08}]  \label{lcth}
Let $(U,V) \in \mathbb{R}^{m \times r}_+ \times \mathbb{R}^{r \times n}_+$ and $M = UV$ with $\rank(M) = \rank(U) = r$. Then the following statments are equivalent: 
\begin{enumerate}
\item[(i)] The exact NMF $(U,V)$ of $M$ is unique (up to permutation and scaling). 
\item[(ii)] There does not exist a non-monomial invertible matrix $Q$ such that $U' = UQ \geq 0$ and $V' = Q^{-1}V \geq 0$. 
\item[(iii)] The polytope $\conv(\theta(U))$ is the unique solution of   NPP($M$) with $r$ vertices. 
\end{enumerate}
\end{theorem}

It is interesting to notice that the columns of $M$ containing zero entries are located on the boundary of the outer polytope $\Delta^{m}$, and these points must be on the boundary of any solution $T$ of NPP($M$).  Therefore, if $M$ contains many zero entries, it is more likely that the set of exact NMF of $M$ will be smaller, since there is less degree of freedom to fill in the space between the inner and outer polytopes. In particular, Donoho and Stodden \cite{DS03} showed that ``requiring that some of the data are spread across the faces of the nonnegative orthant, there is unique simplicial cone'', i.e., there is a unique $\conv(\theta(U))$. \\ 

In the following, based on the assumption that $\rank(M) = \rank_+(M)$, we provide a new uniqueness result using the geometric interpretation of NMF and the sparsity pattern of $M$.  
\begin{lemma} \label{lem1}
Let $M \in \mathbb{R}^{m \times n}$ with $r = \rank(M) = \rank_+(M)$, and $M$ have no all-zero columns. 
If $r$ columns of $\theta(M)$ coincide with $r$ different vertices of $\Delta^{m} \cap \col(\theta(M))$, then the exact NMF of $M$ is unique. 
\end{lemma}
\begin{proof}
Let $(U,V) \in \mathbb{R}^{m \times r}_+ \times \mathbb{R}^{r \times n}_+$ be such that $M = UV$. 
Since $r = \rank(M) = \rank_+(M)$, we must have $\rank(U) = r$ and $\col(U) = \col(M)$ (where $\col(X)$ denotes the column space of matrix $X$), hence 
\[
\conv(\theta(M)) \subseteq \conv(\theta(U)) \subseteq \Delta^{m} \cap \col(\theta(M)). 
\]
Since $r$ columns of $\theta(M)$  coincide with $r$ vertices of $\Delta^{m} \cap \col(\theta(M))$, we have that $\conv(\theta(U)) = \conv(\theta(M))$  is the unique solution of NPP($M$), and Theorem~\ref{lcth} allows to conclude. 
\end{proof} 

In order to identify such matrices, it would be nice to characterize the vertices of $\Delta^{m} \cap \col(\theta(M))$ based solely on the sparsity pattern of $M$. 
By definition, the vertices of $\Delta^{m} \cap \col(\theta(M))$ are the intersection of $r-1$ of its facets, and the facets of $\Delta^{m} \cap \col(\theta(M))$ are given by 
\[
F_i = \{ x \in \Delta^{m} \cap \col(\theta(M)) \ | \ x_i = 0\}. 
\]
Therefore, a vertex of $\Delta^{m} \cap \col(\theta(M))$ must contain at least $r-1$ zero entries. However, this is not a sufficient condition because some facets might be redundant, e.g., if the $i$th row of $M$ is identically equal to zero (for which $F_i = \Delta^{m} \cap \col(\theta(M))$) or if the $i$th and $j$th row of $M$ are equal to each other (for which $F_i = F_j$).  

\begin{lemma} \label{lemv2}
A column of $M$ containing $r-1$ zeros whose corresponding rows have different sparsity patterns corresponds a vertex of $\conv(\theta(M)) \cap \Delta^m$. 
\end{lemma}
\begin{proof}
Let $c$ be one of the columns of $M$ with at least $r-1$ zeros corresponding to rows with different sparsity patterns (i.e., different supports), i.e., there exists $I \subseteq \{ i \ | \ c_i = 0 \}$ with $|I| = r-1$ such that the rows of $M(I,:)$ have different sparsity patterns. 
Let also $F_k = \{ x \ | \ x_{I(k)} = 0 \}$ for $1 \leq k \leq r-1$ denote the $r-1$ facets with $\theta(c) \in F_k$ $\forall k$. To show that $\theta(c)$ is a vertex of $\conv(\theta(M)) \cap \Delta^m$, it suffices to show that the $r-1$ facets are not redundant, i.e., 
for all $1 \leq k < p \leq r-1$,  there exist $x_k$ and $x_p$ in $\conv(\theta(M)) \cap \Delta^m$ such that $x_k \in F_k, x_k \notin F_p$ and $x_p \in F_p, x_p \notin F_k$. 
Because the rows of $M(I,:)$ have different sparsity patterns, for all $1 \leq k < p \leq r-1$, there must exist two indices $h$ and $l$ such that $M(I(k),h) = 0$ and $M(I(p),h) > 0$ while $M(I(k),l) > 0$ and $M(I(p),l) = 0$. Therefore, $\theta(M_{:h}) \in F_k, \theta(M_{:h}) \notin F_p$ and $\theta(M_{:l}) \in F_p, \theta(M_{:l}) \notin F_k$ and the proof is complete. 
\end{proof}

\begin{theorem} 
Let $M \in \mathbb{R}^{m \times n}$ with $r = \rank(M) = \rank_+(M)$. If $M$ has $r$ non-zero columns each having $r-1$ zero entries whose corresponding rows have different sparsity patterns, then the NMF of $M$ is unique. 
\end{theorem} 
\begin{proof}
This follows directly from Lemma \ref{lem1} and \ref{lemv2}.
\end{proof}
Here is an example, 
\[
M =  \left( \begin{array}{ccc}
0   &  1  &   1\\
0   &  0  &   1\\
1   &  0  &   0\\  
1   &  1  &   0\\  
\end{array}\right), 
\] 
with $\rank(M) = \rank_+(M) = 3$ whose unique NMF is $M = MI$. Other examples include matrices containing an $r$-by-$r$ monomial submatrix. 
It is also interesting to notice that this result implies that the only 3-by-3 rank-three nonnegative matrices having a unique exact NMF are the monomial matrices (i.e., permutation and scaling of the identity matrix) since all other matrices have at least two distinct exact NMF: $M = M I = I M$. \\

Finally, the geometric interpretation of NMF shows that sparser matrices $M$ lead to more well-posed NMF problems because many points of the inner polytope in NPP($M$) are located on the boundary of the outer polytope.  Moreover, because the solution $T$ must contain these points, it will have zero entries as well. In particular, assuming $M$ does not contain a zero column, it is easy to check that for $M = UV$ we have
\[
M_{ij} = 0 \quad \Rightarrow \quad \exists k  \text{ such that } U_{ik} = 0. 
\]

\begin{remark}
Having many zero entries in $M$ is not a necessary condition for having an unique  NMF. In fact, Laurberg et al.\@ \cite{LC08} showed that there exist positive matrices with unique NMF. However, for an NMF $(U,V)$ to be unique, the support of each columns of $U$ (resp.\@  row of $V$) cannot be contained in the support of any another column (resp.\@ row) so that each column of $U$ (resp. row of $V$) must have at least one zero entry.  
In fact, assume the support of the $k$th column of $U$ is contained in the support of $l$th column. Then noting 
$\bar{p} = \argmin_{ \{ p | U(p,k)\neq 0 \} }\frac{U(p,l)}{U(p,k)}$, $\epsilon = \frac{U(\bar{p},l)}{U(\bar{p},k)}$, and  
\[
D_{kl} = -\epsilon, \quad D_{ii} = 1 \; \forall i, \quad D_{ij} = 0 \text{ otherwise}, 
\]
one can check that $D^{-1}$ is as follows 
\[
D^{-1}_{kl} = \epsilon, \quad D^{-1}_{ii} = 1 \; \forall i, \quad D^{-1}_{ij} = 0 \text{ otherwise}, 
\]
that is $D^{-1} \geq 0$. Therefore $(UD,D^{-1}V)$ is an equivalent NMF with a different sparsity pattern since $(UD)_{:l} = UD_{:l} = U_{:l} - \epsilon U_{:k} \geq 0$, and   $U_{\bar{p}l} > 0$ while $(UD)_{\bar{p}l} = 0$. 
\end{remark}

\section{Preprocessing for More Well-Posed and Sparser NMF} \label{prep}

In this section, we introduce a completely new approach to obtain more well-posed NMF problems whose solutions are sparser. As it was shown in the previous paragraph, this can be achieved by working with sparser nonnegative matrices.  
Hence, we look for an $n$-by-$n$ matrix $Q$ such that $MQ = {M'}$ is nonnegative, sparse and $Q$ is inverse-positive. In other words, we would like to solve the following problem: 
\begin{equation} \label{preprocNMF} 
\min_{Q \in \mathbb{R}^{n \times n}} ||MQ||_0 \quad \text{ such that } \quad MQ \geq 0 \text{ and } Q^{-1} \geq 0, 
\end{equation}
where $||X||_0$ is the $\ell_0$-`norm' which counts the number of non-zero entries in $X$. Assuming we can solve \eqref{preprocNMF} and obtain a matrix $M' = MQ$, then any NMF $(U,V')$ of $M'$ with $M' \approx UV'$ gives a NMF for $M$. In fact,  
\[
M \; = \; M' Q^{-1} \; \approx \; U {V' Q^{-1}} \; = \; U V, \; \text{ where } V = V' Q^{-1} \geq 0,  
\] 
for which we have 
\[
|| M - U V ||_F 
= || M' Q^{-1} - U {V' Q^{-1}} ||_F  
= || (M'  - U V') Q^{-1} ||_F  
\leq || M'  - U V' ||_F  \; ||Q^{-1} ||_2. 
\]
In particular, if the NMF of $M'$ is exact, then we also have an exact NMF for $M = M' Q^{-1} = U V' Q^{-1} = UV$. The converse direction, however, is not always true. We return to this point in Section~\ref{unirobpre}. 

In the remaining of this section, we propose a way to finding approximate solutions to problem \eqref{preprocNMF}. First, we briefly review some properties of inverse-positive matrices (Section \ref{ipmm}) in order to deal with the constraint $Q^{-1} \geq 0$. Then, we replace the $\ell_0$-`norm' with the $\ell_2$-norm and solve the corresponding optimization problem using constrained linear least squares (Section~\ref{cllssec}). 


\subsection{Inverse-Positive Matrices} \label{ipmm}

In this section, we  recall the definition of three types of matrices: Z-matrices, M-matrices and inverse-positive  matrices, briefly recall how they are related and provide some useful properties. We refer the reader to the book of Berman and Plemmons \cite{BP94} and the references therein for more details on the subject.  
 
\begin{definition}
An $n$-by-$n$ Z-matrix is a real matrix with non-positive off-diagonal entries. 
\end{definition}
\begin{definition}
An $n$-by-$n$ M-matrix is a real matrix of the following form:
\[
A = s I - B, \quad  s > 0, \quad B \geq 0, 
\] 
where the spectral radius\footnote{The spectral radius $\rho(B)$ of a $n$-by-$n$ matrix $B$ is the supremum among all the absolute values of the eigenvalues of $B$, i.e., $\rho(B) = \max_i |\lambda_i(B)|$.}  $\rho(B)$ of $B$ satisfies $s \geq \rho(B)$. 
\end{definition}
It is easy to see that an M-matrix is also a Z-matrix. 
\begin{definition}
An $n$-by-$n$ matrix $Q$ is inverse positive if and only if $Q^{-1}$ exists and $Q^{-1}$ is nonnegative. 
We will note this set $\mathcal{IP}^n$: 
\[
\mathcal{IP}^n = \{ Q \in \mathbb{R}^{n \times n} \ | \ Q \text{ is full rank and } Q^{-1} \geq 0 \}. 
\]
\end{definition} 

It can be shown that inverse-positive Z-matrices are M-matrices: 
\begin{theorem}[{\cite[Theorem 2.3]{BP94}}] \label{berplem} 
Let $A$ be a Z-matrix. Then the following conditions are equivalent : 
\begin{itemize}
\item $A$ is an invertible M-matrix.
\item $A = sI - B$ with $B \geq 0$, $s > \rho(B)$.
\item $A \in \mathcal{IP}^n$, i.e., $A$ is inverse positive. 
\end{itemize}
\end{theorem}

Here is another well-known theorem in matrix theory which will be useful, see, e.g., \cite{T49, DZ09}. 
\begin{definition}
An $n$-by-$n$ matrix $A$ is irreducible if and only if there does not exist an $n$-by-$n$ permutation matrix $P$ such that 
\[
P^T A P = \left( \begin{array}{cc} B & C \\ 0 & D \end{array} \right), 
\]
where $B$ and $D$ are square matrices. 
\end{definition} 

\begin{definition}
An $n$-by-$n$ matrix $A$ is irreducibly diagonally dominant if $A$ is irreducible,  
\begin{equation} \label{dd} \tag{Diagonal Dominance}
|A_{ii}| \geq \sum_{k \neq i} |A_{ki}|, \quad \text{ for } i = 1, 2, \dots, n, 
\end{equation}
and the inequality is strict for at least one $i$. 
\end{definition} 

\begin{theorem} \label{olga} 
If $A$ is irreducibly diagonally dominant, then $A$ is nonsingular. 
\end{theorem}


\subsection{Constrained Linear Least Squares Formulation for \eqref{preprocNMF}} \label{cllssec}

The $\ell_0$-`norm' is of combinatorial nature and typically leads to intractable optimization problems. The standard approach is to use the $\ell_1$-norm instead but we propose here to use the $\ell_2$-norm. The reason is twofold: 
\begin{itemize}
\item When looking at the structure of problem \eqref{preprocNMF}, we observe that any (reasonable) norm will induce solutions with zero entries. In fact, some of the constraints $MQ \geq 0$ will always be active at optimality because of the objective function $||MQ||$.  

\item The $\ell_2$-norm is smooth hence its optimization can be performed more efficiently\footnote{Because of the constraint $MQ \geq 0$, the $\ell_1$-norm problem can actually be decoupled into $n$ linear programs (LP) in $n$ variables and $m+n$ constraints,  and can be solved effectively. 
However, in the noisy case (cf.\@ Section~\ref{noisy}), we would need to introduce $mn$ auxiliary variables (one for each term of the objective function) which turns out to be impractical.}. 

\end{itemize}

We then would like to solve  
\begin{equation} \label{tig1}
\min_{Q \in \mathcal{IP}^n} ||MQ||_F^2 \quad \text{ such that } \quad MQ \geq 0. 
\end{equation} 
Optimizing over the set of inverse-positive matrices $\mathcal{IP}^n$ seems to be very difficult. At least, describing $\mathcal{IP}^n$ explicitly as a semi-algebraic set  requires about $n^2$ polynomial inequalities of degree up to $n$, each with up to $n!$ terms. 
However, we are not aware of a rigorous analysis of the complexity of this type of problems; this is a topic for further research. 


For this reason, we will restrict the search space to the subset of Z-matrices, i.e., inverse-positive matrices of the form $Q = sI-B$, where $s$ is a nonnegative scalar, $I$ is the identity matrix of appropriate dimension and $B$ is a nonnegative matrix such that $\rho(B) < s$, see Section~\ref{ipmm}. 
It is important to notice that 
\begin{itemize} 


\item The scalar $s$ cannot be chosen arbitrarily. In fact, making $s$ go to zero and $B = 0$, the objective function value goes to zero, which is optimal for \eqref{tig1}. 
The same degree of freedom is in fact present in the original problem \eqref{preprocNMF} since $Q$ and $\alpha Q$ for any $\alpha > 0$ are equivalent solutions. Therefore,  
without loss of generality, we fix $s$ to one . 

\item  The diagonal entries of $B$ cannot be chosen arbitrarily. In fact, taking $B$ arbitrarily close (but smaller) to the identity matrix, the infimum of \eqref{tig1} will be equal to zero. We then have to set an upper bound (smaller than one) for the diagonal entries of $B$. It can be checked that this upper bound will always be attained (because of the minimization), and that the optimal solutions corresponding to different upper bounds will be multiples of each other. We therefore fix the bound to zero implying $B_{ii} = 0$ for all $i$, i.e., $Q_{ii} = 1$ for all $i$.  

\end{itemize} 

Finally, we would like to solve  
\[
\min_{Q \in \mathcal{Q}^n} \quad 
 ||MQ||_F^2 \quad  
\text{ such that } \quad  MQ \geq 0, 
\]
where
\[
 \mathcal{Q}^n = \{ Q \in \mathbb{R}^{n \times n} \ | \ Q = I - B, B \geq 0, B_{ii} = 0 \ \forall i, \rho(B) < 1 \} \subset \mathcal{IP}^n. 
\]
Since $MQ = M (I-B) \geq 0$, this problem is equivalent to 
\begin{align}
\min_{B \in \mathbb{R}^{n \times n}} \quad 
& 
\sum_{i=1}^n \; \Big\| M_{:i} - \sum_{k \neq i} M_{:k} B_{ki} \Big\|_2^2 \nonumber \\
\text{ such that } \quad & M \geq MB,  \label{tig2} \\
									 \quad & \rho(B) < 1, \nonumber  \\
									 \quad & B_{ii} = 0 \; \forall i, \; B \geq 0.   \nonumber 
\end{align} 
Without the constraint on the spectral radius of $B$, this is a constrained linear least squares problem (CLLS) in $\mathcal{O}(n^2)$ variables and $\mathcal{O}(n^2 + mn)$ constraints.
The $i$th column of $M'=MQ$, which is the preprocessed version of $M$, 
 will then be given by the following linear combination 
\begin{equation} \label{interpre} 
{M}_{:i}' = M Q_{:i} = M_{:i} - \sum_{k=1}^n M_{:k} B_{ki} \geq 0, \quad \text{ where } B_{ki} \geq 0 \; \forall i, k \; \text{ and } \; B_{ii} = 0. 
\end{equation}
This means that we will subtract from each column of $M$ a nonnegative linear combination of the other columns of $M$ in order to maximize its sparsity while keeping its nonnegativity. 
Intuitively, this amounts to keeping only the non-redundant information from each column of $M$ (see Section~\ref{appl} for some visual examples).

%

\subsubsection{Relaxing the Constraint on the Spectral Radius}

In general, there is no easy way to deal with the non-convex constraint $\rho(B) < 1$. In particular, this constraint may lead to difficult optimization problems, e.g., finding the nearest stable matrix to an unstable one: 
\[ 
\min_X ||X-A|| \quad \text{ such that } \quad  \rho(X) \leq 1, 
\] 
see \cite{PS05} and the references therein. This means that even the projection of the feasible set is non-trivial.

However, we will prove in Section~\ref{propp} that if the columns of $M$ are not multiples of each other, then any optimal solution of problem \eqref{tig2} without the constraint on the spectral radius of $B$, i.e., any optimal solution $B^*$ of
\begin{equation} \label{tigrel}
\min_{B \in \mathbb{R}^{n \times n}_+} 
\quad 
\sum_{i=1}^n \; \Big\| M_{:i} - \sum_{k \neq i} M_{:k} B_{ki} \Big\|_2^2 
\quad \text{ such that } \quad  M \geq MB, \; B_{ii} = 0 \; \forall i, 
\end{equation} 
automatically satisfies $\rho(B^*) < 1$. Hence, the approach may only fail when there are repetitions in the dataset. The reason is that when a column is multiple of another one, say $M_{:i} = \alpha M_{:j}$ for $i\neq j$ and $\alpha > 0$, then taking $B_{ij} = \alpha$ (0 otherwise for that column) gives $MQ_{:i} = M_{:i} - \alpha M_{:j} = 0$ and similarly for $M_{:j}$. Hence we have lost a component in our dataset and potentially produce a lower rank matrix $MQ$. 
In practice, it will be important to make sure that the columns of $M$ are not multiples of each other (even though it is usually not the case for well-constructed datasets).

%
%
%
%
%

\section{Properties of the Preprocessing} \label{propp}


In the remainder of the paper, we denote $\mathcal{B}^*(M)$ the set of optimal solutions of problem \eqref{tigrel} for the data matrix $M$, and $\mathcal{P}$ the  preprocessing operator  defined as  
\[
\mathcal{P}: \mathbb{R}^{m \times n}_+ \to \mathbb{R}^{m \times n}_+ : M \mapsto \mathcal{P}(M) = M(I-B^*), \text{ where } B^* \in \mathcal{B}^*(M). 
\]
In this section, we prove some important properties of $\mathcal{P}$ and $\mathcal{B}^*(M)$: 
\begin{itemize}


 \item The preprocessing operator $\mathcal{P}$ is well-defined (Theorem~\ref{unik}).  
 \item The preprocessing operator $\mathcal{P}$ is invariant to permutation and scaling of the columns of $M$ (Lemma~\ref{permscal}). 
 \item If the columns of $\theta(M)$ are distinct, then $\rho(B^*) < 1$ for any $B^* \in \mathcal{B}^*(M)$ (Theorem~\ref{specrad}). 
\item If the vertices of $\conv(\theta(M))$ are distinct then 
  \begin{itemize}
  \item There exists $B^* \in \mathcal{B}^*(M)$ such that $\rho(B^*) < 1$ (Corollary~\ref{existe}).     
  \item $\rank(\mathcal{P}(M)) = \rank(M)$ and $\rank_+(\mathcal{P}(M)) \geq \rank_+(M)$ (Corollary~\ref{rknrk}). 
  \end{itemize} 

\item 
If the matrix $M$ is separable, then the preprocessing allows to recover a sparse and optimal solution of the corresponding NMF problem (Theorem~\ref{recsep}). In particular it is always optimal for rank-two matrices (Corollary~\ref{rk2}). 

\item 
If the matrix has rank-three, then the preprocessing yields an instance in which the number of solutions of the exact NMF problem is finite (Theorem~\ref{finite3}). 
\end{itemize}

\subsection{General Properties} \label{gp}

A crucial property of our preprocessing is that it is well-defined. 
\begin{theorem} \label{unik}
The preprocessing $\mathcal{P}(M)$ is well-defined, i.e., for any $B_1^* \in \mathcal{B}^*(M), B_2^* \in \mathcal{B}^*(M)$, we have $M(I-B_1^*) = M(I-B_2^*) = \mathcal{P}(M)$. 
\end{theorem}
\begin{proof}
Problem \eqref{tigrel} can be decoupled into $n$ independent CLLS (one for each column of $M$) of the form: 
\begin{equation} \label{tigrelsep}
\min_{b \in \mathbb{R}^{n-1}_+} \| d - C b \|^2  \text{ such that }   Cb \leq d 
\quad \equiv \quad 
\min_{b \in \mathbb{R}^{n-1}_+, y \in \mathbb{R}^m} 
\quad 
\| d - y \|^2  \text{ such that }   y \leq d, y = Cb. 
\end{equation}
The result follows from the fact that the $\ell_2$ projection onto a polyhedral set (actually any convex set) yields a unique point. 
\end{proof}

%

Another important property of the preprocessing is its invariance to permutation and scaling of the columns of $M$. 
\begin{lemma} \label{permscal}
Let $M$ be a nonnegative matrix and $P$ be a monomial matrix. Then, $\mathcal{P}(MP) = \mathcal{P}(M)P$. 
\end{lemma}
\begin{proof} 
We are going to show something slightly stronger; namely that $B^*$ is an optimal solution of \eqref{tigrel} for matrix $M$ if and only if $P^{-1} B^* P$ is an optimal solution of \eqref{tigrel} for matrix $MP$, i.e., 
 \[
 B^* \; \in \; \mathcal{B}^*(M) 
 \quad \iff \quad  
 P^{-1} B^* P \; \in \; \mathcal{B}^*(MP). 
 \]
First, note that  $B$ is a feasible solution of \eqref{tigrel} for $M$ if and only if  $P^{-1} B P$ is a feasible solution of \eqref{tigrel} for $MP$. In fact, nonnegativity of $B$ and its diagonal zero entries are clearly preserved under permutation and scaling while 
\[
M \geq MB \iff M P \geq M B P \iff  M P \geq M (P P^{-1}) B P \iff M P \geq (MP) (P^{-1} B P).
\] 
Hence there is one-to-one correspondence between feasible solutions of \eqref{tigrel} for $M$ and \eqref{tigrel} for $MP$. 

Then, let $B^*$ be an optimal solution of \eqref{tigrel}. 
Because \eqref{tigrel} can be decoupled into $n$ independent CLLS's, one for each column of $B$ (cf.\@ Equation~\eqref{tigrelsep}), we have  
\[
||M_{:i} - M B^*_{:i}||_2^2 \leq ||M_{:i} -  M B_{:i}||_2^2, \quad \forall i, 
\]
for any feasible solution $B$  of \eqref{tigrel}. 
Letting $p \in \mathbb{R}^n_+$ be such that $p_i$ is equal the non-zero entry of the $i$th row of $P$, we have 
\begin{align*}
\sum_i p_i^2 ||M_{:i} - M B^*_{:i}||_2^2 
& = \sum_i ||M_{:i} p_i - M P P^{-1} B^*_{:i} p_i||_2^2 \\
& = ||M P -  M P P^{-1} B^* P||_F^2   \\ 
& \leq \sum_i p_i^2 ||M_{:i} - M B_{:i}||_2^2 = ||M P -  M P P^{-1} B P||_F^2, 
\end{align*}
for any feasible solution $B' = P^{-1} B P$ of \eqref{tigrel} for $MP$. This proves $B^* \in \mathcal{B}^*(M) \Rightarrow  P^{-1} B^* P \in \mathcal{B}^*(MP)$. The other direction follows directly by using the permutation $P^{-1}$ on the matrix $MP$.  
\end{proof}

It is interesting to observe that if a column of $M$ belongs to the convex cone generated by the other columns, 
 then the corresponding column of $\mathcal{P}(M)$ is equal to zero. 

\begin{lemma} \label{lem2} 
Let $\mathcal{I} = \{1,2,\dots,n\} \backslash \{ i \}$. Then $\mathcal{P}(M)_{:i} = 0$ if and only if $M_{:i} \in \cone( M(:, \mathcal{I} ))$. 
\end{lemma}
\begin{proof}
We have that 
\[
\mathcal{P}(M)_{:i} = M_{:i} - \sum_{k \neq i} B^*_{ki} M_{:k} = 0,  \quad B^*_{ki} \geq 0 \iff 
M_{:i} = \sum_{k \neq i} B^*_{ki} M_{:k},  \quad B^*_{ki} \geq 0. 
\] 
\end{proof}


The preprocessed matrix $\mathcal{P}(M)$ may contain all-zero columns, for which the function $\theta(.)$ is not defined (cf.\@ Definition~\ref{pbm}).  We extend the definition to matrices with zero columns as follows:  $\theta(X)$ is the matrix whose columns are the normalized non-zero  columns of $X$, i.e., letting $Y$ be the matrix $X$ where the non-zero columns have been removed, we define $\theta(X) = \theta(Y)$. Hence $\conv(\theta(X))$ denotes the convex hull of the normalized non-zero columns of $X$. 

Another straightforward property is that the preprocessing can only inflate the convex hull defined by the columns of $\theta(M)$. 
\begin{lemma} \label{lem3} 
Let $M \in \mathbb{R}^{m \times n}_+$. If the vertices of $\conv(\theta(M))$ are non-repeated, then 
\[
\conv(\theta(M)) \quad \subseteq \quad \conv(\theta(\mathcal{P}(M))) \quad \subseteq \quad \Delta^m \cap \col(\theta(M)). 
\]
\end{lemma}
\begin{proof} 
By construction, since $\mathcal{P}(M) = MQ$, $\col(\theta(\mathcal{P}(M))) \subseteq \col(\theta(M))$ and $\conv(\theta(\mathcal{P}(M))) \subseteq  \Delta^m \cap \col(\theta(M))$. 
Let $i$ be the index corresponding to a vertex of $\theta(M)$ and $\mathcal{I} = \{1,2,\dots,n\} \backslash \{ i \}$. Because vertices of $\theta(M)$ are non-repeated, we have $M_{:i} \notin \conv(\theta(M(:,\mathcal{I})))$, while 
\[
\mathcal{P}(M)_{:i} = M_{:i} - \sum_{k \neq i} b_{ki} M_{:k} 
\quad \iff \quad 
M_{:i} = \mathcal{P}(M)_{:i} + \sum_{k \neq i} b_{ki} M_{:k}. 
\]
Hence $M_{:i} \in \conv(\theta( [\mathcal{P}(M)_{:i} \, M(:,\mathcal{I})] ))$, which implies that 
\[
\conv(\theta( M ) ) \subseteq \conv(\theta( [\mathcal{P}(M)_{:i} \, M(:,\mathcal{I})] )), 
\] 
so that replacing $M_{:i}$ by $\mathcal{P}(M)_{:i}$ extends $\conv(\theta(M))$. Since this holds for all vertices, the proof is complete.  
\end{proof}

\begin{corollary} \label{rknrk} 
Let $M \in \mathbb{R}^{m \times n}_+$. 
If no column of $M$ is multiple of another column, then 
\[
\rank(\mathcal{P}(M)) = \rank(M) \quad \text{ and } \quad \rank_+(\mathcal{P}(M)) \geq \rank_+(M). 
\]
\end{corollary}
\begin{proof} 
Without loss of generality, we can assume that $M$ does not have a zero column. In fact, a preprocessed zero column remains zero while it cannot influence the preprocessing of the other columns (see Equation~\eqref{interpre}).  
Then, by Lemma~\ref{lem3}, we have 
\[
\conv(\theta(M)) 
\quad \subseteq \quad 
\conv(\theta(\mathcal{P}(M))) 
\quad \subseteq \quad 
\Delta^m \cap \col(\theta(M)), 
\]
implying $\rank_+(\mathcal{P}(M)) \geq \rank_+(M)$ and $\rank(\mathcal{P}(M)) = \rank(M)$. 

Another way to prove this result is to use Corollary~\ref{existe} (see below) guaranteeing the existence of an inverse-positive matrix $Q$ such that $\mathcal{P}(M) = MQ$ which implies $\rank(\mathcal{P}(M)) = \rank(M)$. Moreover, any exact NMF $(U,V) \in \mathbb{R}^{m \times r} \times \mathbb{R}^{r \times n}$ of  $\mathcal{P}(M)$ gives $M = UVQ^{-1}$ hence $\rank_+(M) \leq \rank_+(\mathcal{P}(M))$. 
\end{proof}

We now prove that if no column of $M$ is multiple of another column (i.e., the columns of $\theta(M)$ are distinct) then $\rho(B^*) < 1$ for any $B^* \in \mathcal{B}^*(M)$ whence $Q = I - B^*$ is an inverse positive matrix. 

\begin{lemma} \label{diagdom}
Let $A$ be a column stochastic matrix and $Q = I-B$ where $B\geq0$ and $B_{ii} = 0$ for all $i$ be such that $AQ \geq 0$. Then,   
\[
\sum_k B_{ki} \leq 1, \quad \forall i, 
\]
i.e., $Q$ is diagonally dominant. Moreover, 
if $A_{:i} \notin \conv(A(:,\mathcal{I}))$ where $\mathcal{I} = \{1,2,\dots,n\} \backslash \{ i \}$,  then 
\[
\sum_k B_{ki} < 1. 
\]
\end{lemma}
\begin{proof}
By assumption, we have for all $i$
\[
A_{:i} \geq A B_{:i} = \sum_k A_{:k} B_{ki} 
\quad   \Rightarrow \quad  
1 = ||A_{:i}||_1 \geq ||AB_{:i}||_1 = ||\sum_k A_{:k} B_{ki}||_1 = ||B_{:i}||_1 = \sum_k B_{ki}, 
\]
because $A$ and $B$ are nonnegative. Moreover, if $A_{:i} \notin \conv(A(:,\mathcal{I}))$, then there exists at least one index $j$ such that $A_{ji} > A_{j:} B_{:i}$ (Lemma~\ref{lem2}) so that the above inequality is strict. 
\end{proof}

\begin{theorem} \label{specrad}
If no column of $M$ is multiple of another column, then any optimal solution $B^*$ of \eqref{tigrel} satisfies $\rho(B^*) < 1$, i.e., $Q = I - B^*$ is inverse positive. 
\end{theorem}
\begin{proof} 
By Theorem~\ref{berplem}, $\rho(B^*) < 1$ if and only if $Q = I - B^*$ is inverse positive if and only if $Q$ is a nonsingular M-matrix. 
Let us then show that $Q$ is a nonsingular M-matrix. 
First, we can assume without loss of generality that
\begin{itemize}
\item 
Matrix $M$ does not contain a column equal to zero. In fact, if $M$ does, say the first column is equal to zero, then we must have $B_{:1} = 0$ (since $M_{:1} \geq MB_{:1}$ and there is not other zero column in $M$). The matrix $Q$ is then a nonsingular M-matrix if and only if $Q(2$:$n,2$:$n)$ is.  

\item 
The columns of $M$ sum to one. In fact, letting $P=D(M)$ be defined as in Equation~\eqref{diago}, by Lemma~\ref{permscal}, $B^*$ is an optimal solution for $M$ if and only if $P^{-1}B^*P$ is an optimal solution for $MP$. Since $B^*$ and $P^{-1}B^*P$ share the same eigenvalues, $\rho(B^*) < 1 \iff \rho(P^{-1}B^*P) < 1$.

\item 
Let $B \in \mathcal{B}^*(M)$, $Q = I-B^*$, and $P$ be a permutation matrix  such that 
\[
P^T Q P 
= \left( \begin{array}{ccccc} 
Q^{(1)} & Q^{(12)} &  Q^{(13)} & \dots & Q^{(1k)} \\
0 & Q^{(2)} &  Q^{(23)} & \dots & Q^{(2k)} \\
0 & 0 & Q^{(3)} & \dots & Q^{(3k)} \\
\vdots &  \dots & \ddots & \ddots & \vdots \\
0 & \dots &  \dots & 0 & Q^{(k)} \\
\end{array} \right) 
= I - 
\left( \begin{array}{ccccc} 
B^{(1)} & B^{(12)} &  B^{(13)} & \dots & B^{(1k)} \\
0 & B^{(2)} &  B^{(23)} & \dots & B^{(2k)} \\
0 & 0 & B^{(3)} & \dots & B^{(3k)} \\
\vdots &  \dots & \ddots & \ddots & \vdots \\
0 & \dots &  \dots & 0 & B^{(k)} \\
\end{array} \right), 
\]
where $Q^{(i)}$ are irreducible for all $i$.  Without loss of generality, by Lemma~\ref{permscal}, we can then assume that $Q$ has this form. 
\end{itemize} 


In the following we show that $Q^{(p)}$ is nonsingular for each $1 \leq p \leq k$ hence $Q$ is. 
By Theorem~\ref{olga}, if  $Q^{(p)}$ is irreducibly diagonally dominant, then $Q^{(p)}$ is nonsingular and the proof is complete. We already have that $Q^{(p)}$ is irreducible for $1 \leq p \leq k$. 
Let $I_p$ denote the index set such that $Q^{(p)} = Q(I_p,I_p)$. We have $M(I_p,:)$ is column stochastic, and 
\[
\mathcal{P}(M)(I_p,:) = M(I_p,:) -  \sum_{l=1}^{p-1} M(I_l,:) B^{(lp)} - M(I_p,:) B^{(p)} \geq 0, 
\]
implying that $M(I_p,:) \geq M(I_p,:) B^{(p)}$.  Moreover the columns of $M(I_p,:)$ are distinct so that there is at least one which does not belong to the convex hull of the others. Hence, by Lemma~\ref{diagdom}, $Q^{(p)}$ is irreducibly diagonally dominant. 
\end{proof}

\begin{corollary} \label{existe}
Let $M \in \mathbb{R}^{m \times n}_+$. If the vertices of $\conv(\theta(M))$ are non-repeated, then there exists one optimal solution $B^* \in \mathcal{B}^*(M)$ such that $\rho(B^*) < 1$, i.e., such that $Q = I-B^*$ is an inverse-positive matrix. 
\end{corollary}
\begin{proof}
Let us show that $Q$ is a nonsingular M-matrix. First, by Lemma~\ref{diagdom}, $Q$ is diagonally dominant implying $\rho(B) \leq 1$ so that $Q$ is an M-matrix (cf. Theorem~\ref{specrad}). We can assume without loss of generality that the $r$ first  columns of $M$ correspond to the vertices of $\conv(\theta(M))$. This implies that there exists an optimal solution $B^* \in \mathcal{B}^*(M)$ with the following form
\[
\left( \begin{array}{cc} Q_1 & Q_{12} \\ 0 & I \end{array} \right) 
= I - \left( \begin{array}{cc} B_1^* & B^*_{12} \\ 0 & 0 \end{array} \right). 
\]
In fact, by assumption, the last columns of $M$ belong to the convex cone of the $r$ first ones and can then be set to zero (which is optimal) using only the first $r$ columns (cf.\@ Lemma~\ref{lem2}). Lemma~\ref{diagdom} applies on matrix $Q_1$ and $M(:,1\text{:}r)$ since  
\[
MQ(:,1\text{:}r) = M(:,1\text{:}r) - M(:,1\text{:}r) B_1^* \geq 0, 
\]
while by assumption no column of $M(:,1\text{:}r)$ belong to the convex hull of the other columns, so that $Q_1$ is strictly diagonally dominant hence is a nonsingular M-matrix. 
\end{proof}

Finally, what really matters is that the vertices of $\conv(\theta(M))$ are non-repeated. 
In that case, the preprocessing is unique and the preprocessed matrix has the same rank as the original one. The fact that $Q$ could be singular is not too dramatic. In fact, given an NMF $(U,V')$ of the preprocessed matrix $\mathcal{P}(M) = MQ \approx UV'$, we can obtain the optimal factor $V$ for matrix $M$ by solving the nonnegative least squares problem $V = \argmin_{X \geq 0} ||M-UX||_F^2$ (instead of taking $V = V'Q^{-1}$) and obtain $M \approx UV$. \\

%


\subsection{Recovery under Separability} \label{recusep}

\begin{definition}[Separability] A nonnegative factorization $M = UV$ is called separable if for each $i$ there is some column $f(i)$ of $V$ that has a single nonzero entry and this entry is in the $i$th row, i.e., $V$ contains a monomial submatrix. In other words, each column of $U$ appears (up to a scaling factor) as a column of $M$. 
\end{definition}

Note that this assumption is equivalent to the pure pixel assumption in hyperspectral imaging (i.e., for each constitutive material present in the image, there is at least one pixel containing only that material) \cite{C94} or, in document classification (see Section~\ref{uniqueNMF}), to the assumption that, for each topic, there is at least one document corresponding only to that topic (or, considering the matrix transpose, that there is at least one word corresponding only to that topic \cite{AGKM11}).  


Geometrically, this means that the vertices of $\conv(\theta(M))$ are given by the columns of $\theta(U)$. 
We have the following straightforward lemma: 
\begin{lemma} \label{sepgeo}
$M = UV$ is separable if and only if $\conv(\theta(M)) = \conv(\theta(U))$. 
\end{lemma}
\begin{proof}
$M=UV$ is separable if and only if $V$ contains a monomial submatrix if and only if the vertices of $\theta(U)$ and $\theta(M)$ coincide if and only if $\conv(\theta(M)) = \conv(\theta(U))$. 
\end{proof}


\begin{theorem} \label{recsep}
If $M$ is separable and the $r$ vertices of $\theta(M)$ are non-repeated, then $\mathcal{P}(M)$ has $r$ non-zero columns, say $S_{:1}, S_{:2}, \dots S_{:r}$, such that  $\conv(\theta(M)) \subseteq \conv(\theta(S))$, i.e., there exists $R \geq 0$ such that $M = SR$. 
\end{theorem}
\begin{proof} 
This is a consequence of Lemmas~\ref{lem2} and \ref{sepgeo}. 
\end{proof} 
Theorem~\ref{recsep} shows that the preprocessing is able to identify the $r$ columns of $M=UV$ corresponding to the vertices of $\theta(M)$. Moreover, it returns a sparser matrix $S$, namely $\mathcal{P}(U)$, whose cone contains the columns of $M$. 
Remark also that Theorem~\ref{recsep} does not require $M$ to be full rank, i.e., the dimension of $\conv(\theta(M))$ can be smaller than $r-1$. 


\begin{corollary} \label{rk2}
For any rank-two nonnegative matrix $M$ whose columns are not multiples of each other,  $\mathcal{P}(M)$ has only two non-zero columns, say $S_{:1}$ and  $S_{:2}$ such that $\conv(\theta(M)) \subseteq \conv(\theta(S))$, i.e., there exists $R \geq 0$ such that $M = SR$. In other words, the preprocessing technique is optimal as it is able to identify an optimal nonnegative basis for the NMF problem corresponding to the matrix $M$.  
\end{corollary}
\begin{proof}
A rank-two nonnegative matrix is always separable. In fact, a two-dimensional pointed cone is always spanned by two extreme vectors. In particular $\rank(M) = 2 \iff \rank_+(M) = 2$, see, e.g.,~\cite{Tho}. 
\end{proof}


\begin{example} \label{sepex}
Here is an example with a rank-three separable matrix 
\begin{equation} \label{rk3sep}
M =  \left( 
\begin{array}{cccccccccc} 
5&5&5&5&9&1&4&1&7&7 \\
10&6&5&3&7&8&4&1&5&8 \\
8&9&9&4&7&8&3&9&6&7 \\ 
\end{array} 
\right)^T \left( 
\begin{array}{cccccccc} 
1&0&0&2&3&6&4&4\\
0&1&0&5&7&7&7&4\\
0&0&1&9&4&4&8&6\\ 
\end{array} 
\right). 
\end{equation}
Its preprocessed version is 
\[
\mathcal{P}(M) =  \left( 
\begin{array}{cccccccccc} 
3.6&3.85&3.93&4.29&7.61&  0&3.32&0.48&5.93&5.66\\
6.27&2.54&1.62&0&1.48&6.49&1.48& 0&0.72&3.44\\
0.8&2.4&2.67&0.67&0.67&1.78&0&4.2&0.93&0.62 \\ 
\end{array} 
\right)^T \left( 
\begin{array}{cccccccc} 
1&0&0&0&0&0&0&0\\
0&1&0&0&0&0&0&0\\
0&0&1&0&0&0&0&0\\ 
\end{array} 
\right). 
\]
Figure~\ref{sepa} shows the geometric interpretation of the preprocessing. 
Notice that the preprocessing makes the solution to the corresponding NMF problem unique. 
\begin{figure}[ht!]
\begin{center}
\includegraphics[width=10cm]{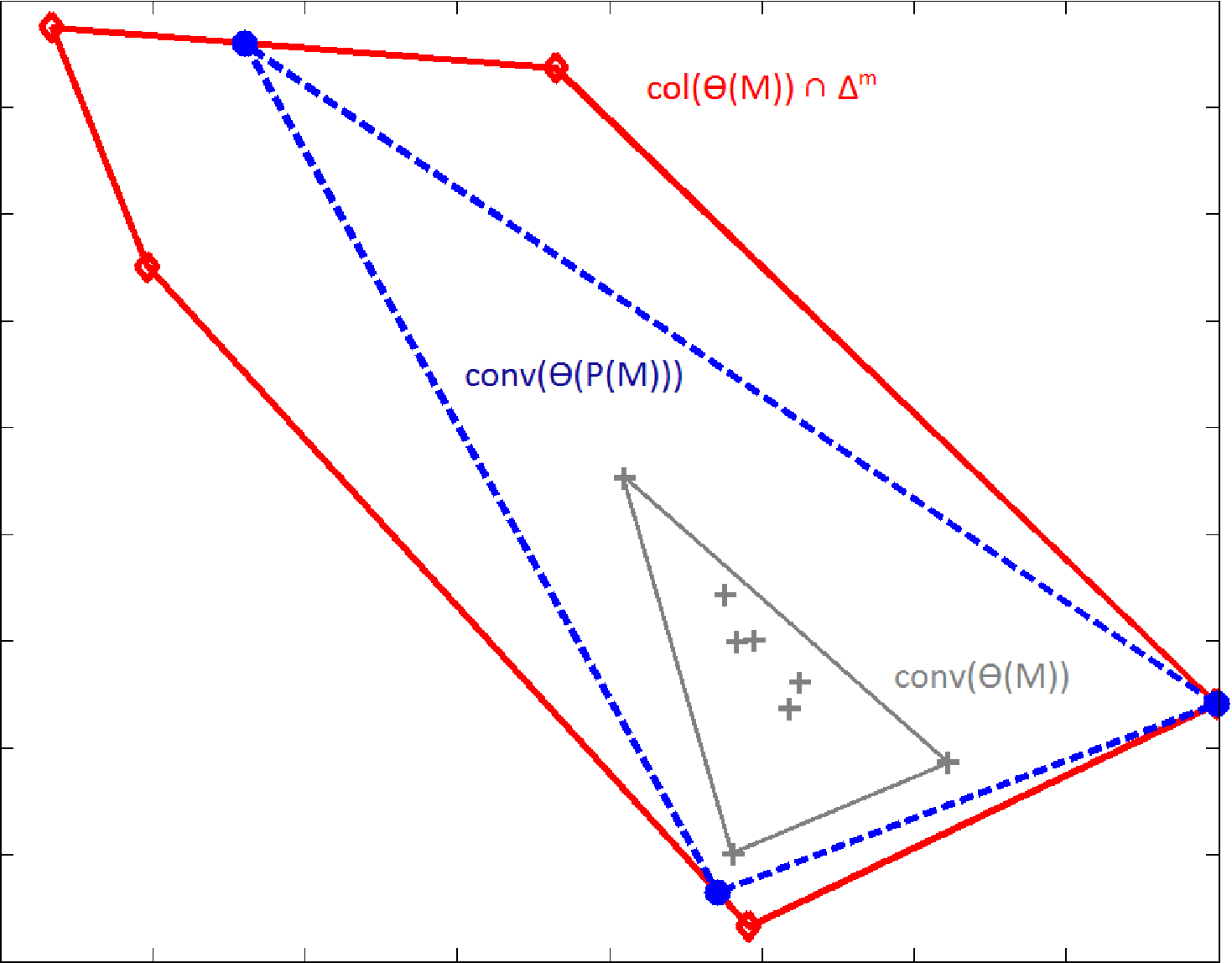}
\caption{Geometric interpretation of the preprocessing of matrix $M$ from Equation~\eqref{rk3sep}.} 
\label{sepa}
\end{center}
\end{figure} 
\end{example}

\subsection{Uniqueness and Robustness Through Preprocessing} \label{unirobpre}

A potential drawback of the preprocessing is that it might increase the nonnegative rank of $M$. In this section, we show how to modulate the preprocessing to prevent this behavior.  \\

Let us define 
\[
\mathcal{P}^{\alpha}(M) = M (I-\alpha B^*)  =  M - \alpha M  B^*, 
\] 
where $0 \leq \alpha \leq 1$ and $B^* \in \mathcal{B}^*(M)$. Notice that $\mathcal{P}^{\alpha}(M)$ is well-defined 
because for any $B_1^*, B_2^* \in \mathcal{B}^*(M)$ we have $M B_1^* = M B_2^*$; see Theorem~\ref{unik}.

\begin{lemma} \label{alphabeta}
Let $M$ be a nonnegative matrix such that the vertices of $\conv(\theta(M))$ are non-repeated.  Then, for any $0 \leq \alpha \leq \beta \leq 1$, 
\[
\conv( \theta(M) ) 
\subseteq \conv( \theta(  \mathcal{P}^{\alpha}(M) ) )
\subseteq \conv( \theta(  \mathcal{P}^{\beta}(M) ) ) 
\subseteq \col( \theta(M) ) \cap \Delta^m. 
\]
Therefore, 
\[
\rank_+(M) \leq \rank_+( \mathcal{P}^{\alpha}(M) ) \leq \rank_+( \mathcal{P}^{\beta}(M) ).
\]
\end{lemma}  
\begin{proof}
The proof can be obtained by following exactly the same steps as the proof of Lemma~\ref{lem3}. 
\end{proof}


\begin{lemma}  \label{usc}
Let $M$ be a nonnegative matrix such that the vertices of $\conv(\theta(M))$ are non-repeated, then the supremum 
\begin{equation} \label{supr}
\bar{\alpha} = \sup_{0 \leq \alpha \leq 1} \alpha \quad \text{ such that } \quad \rank_+( \mathcal{P}^{\alpha}(M) ) = \rank_+( M ) 
\end{equation}
is attained. 
\end{lemma}
\begin{proof}
We can assume without loss of generality that $M$ does not have all-zero columns. In fact, if $M_{:i} = 0$ for some $i$ then  $\mathcal{P}^{\alpha}(M)_{:i} = 0$ for all $\alpha \in [0,1]$ so that the nonnegative rank of $\mathcal{P}^{\alpha}(M)$ is not affected by the zero columns of $M$. 

Then, if $\alpha = 1$, the proof is complete. Otherwise, one can easily check that, for any $0 \leq \alpha < 1$, we have $\mathcal{P}^{\alpha}(M)_{:i} \neq 0$ $\forall i$ (using the same argument as in Lemma~\ref{lem2}). 

Finally, the result follows from the upper-semicontinuity of the nonnegative rank \cite[Theorem 3.1]{BCR10}: `If $P$ is nonnegative matrix, without zero columns and with $\rank_+(P) = k$, then there exists a ball $\mathcal{B}(P,\epsilon)$ centered at $P$ and of radius $\epsilon > 0$ such that $\rank_+(N) \geq k$ for all $N \in \mathcal{B}(P,\epsilon)$'. 
Therefore, if the supremum of \eqref{supr} was not attained, the matrix $\mathcal{P}_{\bar{\alpha}}(M)$ would satisfy $\rank_+( \mathcal{P}_{\bar{\alpha}}(M) ) > \rank_+( M )$ while for any $\alpha < \bar{\alpha}$ we would have $\rank_+( \mathcal{P}_{{\alpha}}(M) ) = \rank_+( M )$, a contradiction.  
\end{proof}

Hence working with matrix $\mathcal{P}^{\bar{\alpha}}(M)$ instead of $M$ will reduce the number of solutions of the NMF problem while preserving the nonnegative rank: 
\begin{theorem} \label{morewellposed} 
Let $M$ be a nonnegative matrix for which the vertices of $\conv(\theta(M))$ are non-repeated, let also $\bar{\alpha}$ be defined as in Equation~\eqref{supr}. Then any NMF $(U,V)$ of $\mathcal{P}^{\bar{\alpha}}(M)$ corresponds to an NMF $(U,VQ^{-1})$ of $M$, while the converse is not true. In fact, 
\[
\conv(\theta(M)) \subseteq \conv(\theta(\mathcal{P}^{\bar{\alpha}}(M))).
\] 
Therefore, the NMF problem for $\mathcal{P}^{\bar{\alpha}}(M)$  is more well posed. 
\end{theorem} 
\begin{proof}
This follows directly from the definition of $\bar{\alpha}$, and Lemmas~\ref{alphabeta} and \ref{usc}. 
\end{proof}

We now illustrate Corollary~\ref{morewellposed} on a simple example, which will lead to three other important results. 
\begin{example}[Nested Squares] \label{2sqex} Let 
\[
M = \left(  \begin{array}{cccc} 
5   &  3  &   3  &   5 \\
3   &  5   &  5  &   3 \\
5   &  5   &  3  &   3 \\
3   &  3   &  5  &   5 \\
\end{array} 
\right). 
\]
The problem NPP($M$) restricted to the column space of $M$ is made up of two nested squares, $\conv( \theta(M) )$ and $\col( \theta(M) ) \cap \Delta^m$, centered at $(0,0)$ with side length 2 and 8  respectively, see Figure~\ref{2sq}. The polygon corresponding to $\mathcal{P}^{\alpha}(M)$ is a square centered at $(0,0)$ with side length depending on $\alpha$, between 2 (for $\alpha = 0$) and 8 (for $\alpha = 1$). We can show that the largest such square still included in a triangle corresponds to 
\begin{equation} \label{2sqe}
\mathcal{P}^{\bar{\alpha}}(M) =  
\mathcal{P}^{\bar{\alpha}} \left(  \begin{array}{cccc} 
5   &  3  &   3  &   5 \\
3   &  5   &  5  &   3 \\
5   &  5   &  3  &   3 \\
3   &  3   &  5  &   5 \\
\end{array} 
\right)   
= \frac{1}{a} 
\left( \begin{array}{cccc} 
    1+a  &  1-a  &  1-a  &  1+a\\
    1-a  &  1+a  &  1+a  &  1-a\\
    1+a   & 1+a  &  1-a  &  1-a\\
    1-a  &  1-a &   1+a  &  1+a \\
    \end{array} 
\right), 
\end{equation}
where $a = \sqrt{2}-1$ and $\bar{\alpha} = \frac{4a-1}{3a}$ (this follows from the proof of Theorem~\ref{finite3}; see below). Hence, the polygon $\conv( \theta(\mathcal{P}^{\bar{\alpha}}(M)) )$ is a square centered at $(0,0)$ with side length $8a$ in between $\conv( \theta(M) )$ and $\col( \theta(M) ) \cap \Delta^m$, see Figure~\ref{2sq}. 
\begin{figure}[ht!]
\begin{center}
\includegraphics[width=12cm]{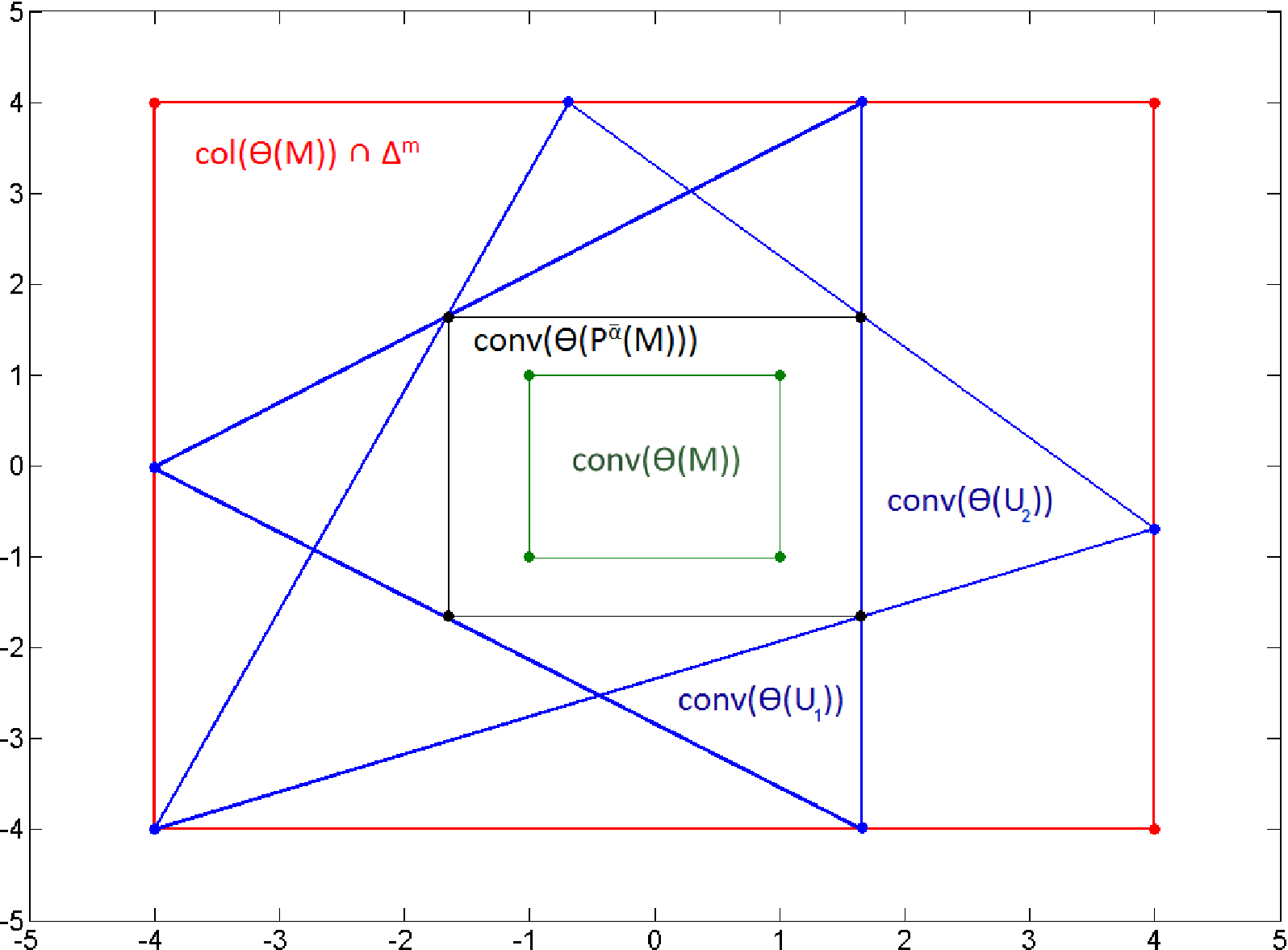}
\caption{Geometric interpretation of the preprocessing of matrix $M$ from Equation~\eqref{2sqe}.} 
\label{2sq}
\end{center}
\end{figure} 
Unfortunately, the NMF of $\mathcal{P}^{\bar{\alpha}}(M)$ is non-unique. In fact, we will see later that it has 8 solutions (the ones drawn on Figure~\ref{2sq} and their rotations). 

%
%
\end{example}

Example \ref{2sqex} illustrates the following three important facts: \\ 

\noindent \textbf{Fact 1. Defining a well-posed NMF problem is not always possible.} 
In other words, \emph{there does not exist any `reasonable' NMF formulation having always a unique solution (up to permutation and scaling).} 
In fact, Example~\ref{2sqex} shows that, because of the symmetry of the problem, any solution of NPP($M$) can be rotated by 90, 180 or 270 degrees to obtain a different solution with exactly the same characteristics (the rotated solutions cannot be distinguished in any reasonable way).  
For example, there are 4 solutions which are the sparsest, each containing one vertex of $\col( \theta(M) ) \cap \Delta^m$, see $\conv( \theta(U_2) )$ on Figure~\ref{2sqex}, including  
\[
U_2 = \left( \begin{array}{ccc} 
    1 &  a     &  0  \\
    0  &  1-a  &  1  \\
    a  & 1  &  0 \\
    1-a  &  0 &   1   \\
    \end{array} 
\right), \text{ and its rotation of 180 degrees } 
U_2^{(180)} = \left( \begin{array}{ccc} 
    0 &  1-a    &  1  \\
    1  &  a  &  0  \\
    1-a     & 1  &  1 \\
    a   &  0 &   0   \\
    \end{array} 
\right). \vspace{0.4cm} \\
\]

\noindent \textbf{Fact 2. The preprocessing makes NMF more robust.} For any $m$-by-$n$ matrix $E$ such that $\col(E) \subseteq \col(M)$, $M+E \geq 0$, and 
\[
\conv(\theta(M)) \subseteq \conv(\theta(M+E)) \subseteq \conv(\mathcal{P}^{\bar{\alpha}}(M)), 
\]
the exact NMF $(U,V)$ of $\mathcal{P}^{\bar{\alpha}}(M)$ will still provide an optimal factor $U$ for the perturbed matrix $M+E$. In particular, if the matrix $M$ is positive, then one can show that\footnote{Using the same ideas as in Lemma~\ref{lem3} and the fact that any preprocessed column must contain at least one zero entry.}  $\conv(\theta(M))$ is strictly contained in $\conv(\mathcal{P}^{\bar{\alpha}}(M))$ (given that $\bar{\alpha} > 0$) so that any sufficiently small perturbation $E$ with $\col(E) \subseteq \col(M)$ will satisfy the conditions above. 

In Example~\ref{2sqex}, the vertices of $M$ can be perturbed and, as long as they remain inside the square defined by $\conv(\mathcal{P}^{\bar{\alpha}}(M))$ (see Figure~\ref{2sq}), the exact NMF of $\conv(\mathcal{P}^{\bar{\alpha}}(M))$ will provide an exact NMF for the perturbed matrix $M$. (More precisely, any matrix $E$ such that $\col(E) \subseteq \col(M)$ and $\max_{i,j} |E_{ij}| \leq \sqrt{2}-1$ will satisfy $\conv(\theta(M+E)) \subseteq \conv(\mathcal{P}^{\bar{\alpha}}(M))$.) \\ 


\noindent \textbf{Fact 3. The preprocessing reduces the number of solutions of the NMF problem. } In Example~\ref{2sqex}, even though the NMF of $\mathcal{P}^{\bar{\alpha}}(M)$ is non-unique, the set of solutions has been drastically reduced: from a two-dimensional space to a zero-dimensional one containing eight points ($\conv(\theta(U_1))$, $\conv(\theta(U_2))$ and the corresponding rotated solutions, see Figure~\ref{2sqex}).

\begin{theorem} \label{finite3}
Let $M \in \mathbb{R}^{m \times n}_+$ be such that $\rank(M) = \rank_+(M) = 3$ and let $\bar{\alpha}$ be defined as in Equation~\eqref{supr}. 
Assume also that $\conv(\theta(\mathcal{P}(M)))$ 
has at least four vertices. 
Then the number of solutions of NPP($\mathcal{P}^{\bar{\alpha}}(M)$) with three vertices is smaller than $m+n$.  
\end{theorem} 
\begin{proof}
Let $P$ and $Q$ denote the outer and inner polygons of NPP($\mathcal{P}^{\bar{\alpha}}(M)$), respectively. 
Let us also parametrize the boundary of the outer polygon $P$ with the parameter $t \in [0,1]$ and the function 
\[
x : \mathbb{R}_+ \to \mathbb{R}^{2} : t \mapsto x(t) \in P,  
\]
where $x$ is a continuous function with $x(0) = x(1)$ and $\{ x(t) \ | \ t \in [0,1] \}$ is equal to the boundary of $P$. We also define the function $x$ for values of $t$ larger than one using $x(t) = x( t - \left\lfloor  t \right\rfloor$) where $\left\lfloor  t \right\rfloor$ is the largest integer not exceeding $t$. Using the construction of 
Aggarwal et al.\@~\cite{ABOS89}, we define the function 
\[
f_k : \mathbb{R}_+ \to \mathbb{R}_+ : t \mapsto f_k(t)
\] 
as follows. Let $t_1 \in [0,1)$ and $x(t_1)$ be the corresponding point on the boundary of $P$. From $x(t_1)$, we can trace the tangent to $Q$ (i.e., $Q$ is on one side of the tangent, and the tangent touches $Q$), say in the clock-wise direction, intersect it with $P$ and hence obtain a new point $x(t_2)$ on the boundary $P$ (see Figure~\ref{2sqf} for an illustration on the nested squares problem). 
\begin{figure}[ht!]
\begin{center}
\includegraphics[width=12cm]{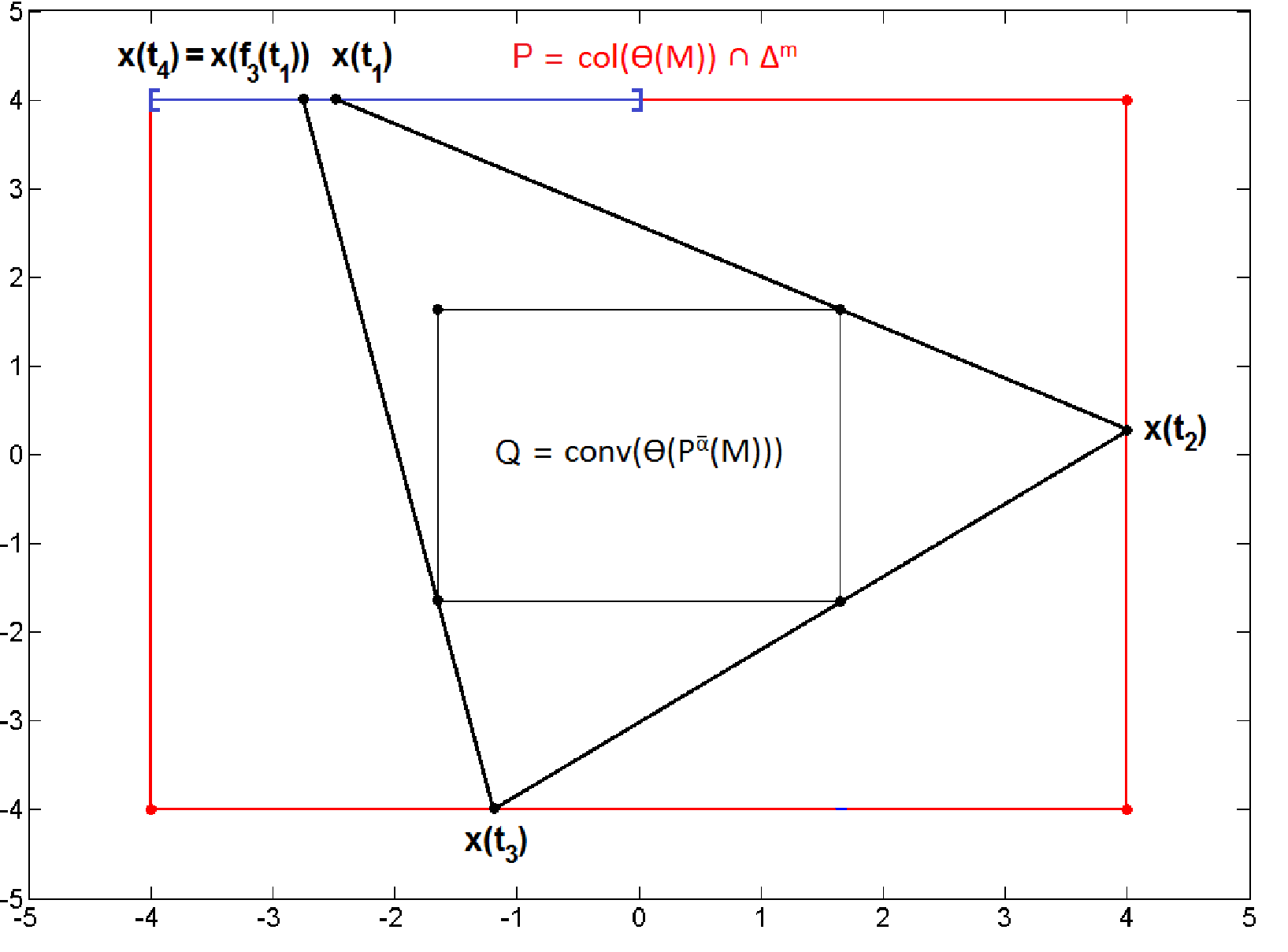}
\caption{Mapping of the point $x(t_1)$ to $x(t_4)$ using the construction from \cite{ABOS89}.} 
\label{2sqf}
\end{center}
\end{figure}
We assume without loss of generality that $t_2 \geq t_1$ (i.e., if $t_2$ happens to be larger than one, we do not round it down with the equivalent value $t_2 - \left\lfloor  t_2 \right\rfloor$). 
Starting from $x(t_2)$, we can use the same procedure to obtain $x(t_3)$ and we apply this procedure $k$ times to obtain the point $x(t_{k+1})$, where $t_{k+1} \geq \dots \geq t_2 \geq t_1$. Finally, we define $f_k(t_1) = t_{k+1}$. 

Aggarwal et al.\@ showed that $x(t_1)$ can be taken as a vertex of a feasible solution of NPP($\mathcal{P}^{\bar{\alpha}}(M)$) with $k$ vertices if and only if $f_k(t_1) = t_{k+1} \geq t_1 + 1$, i.e., we were able to turn around $Q$ inside $P$ in $k+1$ steps 
(in fact, $x(t_1)$, $x(t_1)$, $\dots$, and $x(t_k)$ are the vertices of a feasible solution). 

Aggarwal et al.\@~\cite{ABOS89} also showed that the function $f_k$ is continuous, non-decreasing, and depends continuously on the vertices of $Q$ (see also Appendix~\ref{Appa}). Figure~\ref{2sqfail} displays the function $f_4$ for the nested squares (Example~\ref{2sqex}).  
\begin{figure}[ht!]
\begin{center}
\includegraphics[width=8cm]{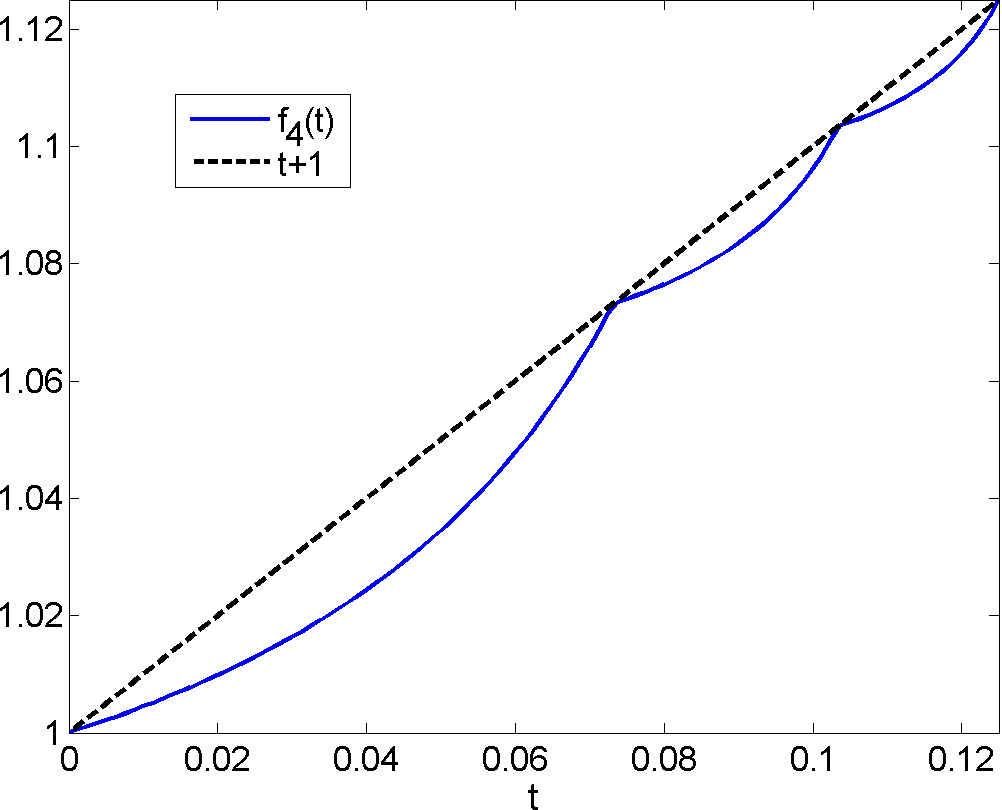}
\caption{Function $f_4(t)$ for Example~\ref{2sqex} using the construction from \cite{ABOS89} (see also Figure~\ref{2sqfail} and Appendix~\ref{Appa}). We only plot the function $f_4$ in the interval $[0,\frac{1}{8}]$ because, by symmetry, $f_4(x+\frac{1}{8}) = f_4(x) + \frac{1}{8}$.} 
\label{2sqfail}
\end{center}
\end{figure}

If $\col(\theta(M)) \cap \Delta^m$ has three vertices, then $\bar{\alpha} = 1$. In fact, we have that 
\[
\theta(\mathcal{P}^{\alpha}(M)) \subseteq \col(\theta(M)) \cap \Delta^m \quad \text{ for any } \; 0 \leq \alpha \leq 1, 
\] 
implying $\rank_+(\mathcal{P}^{\alpha}(M)) = 3$ for all $0 \leq \alpha \leq 1$. 
Moreover, because $\theta(\mathcal{P}^{\alpha}(M))$ has at least four vertices, $\col(\theta(M)) \cap \Delta^m$ is the unique solution of the corresponding NPP problem: the outer polygon is a triangle while the inner polygon has at least four vertices which are located on the edges of the outer triangle (since $\bar{\alpha} = 1$ and each column of $\mathcal{P}(M)$ contains at least one zero entry). 

Let us then assume that $\col(\theta(M)) \cap \Delta^m$ has at least four vertices. We show that this implies  $\bar{\alpha} < 1$. Assume $\bar{\alpha} = 1$. The polygons $P = \col(\theta(M)) \cap \Delta^m$ and $Q=\theta(\mathcal{P}(M))$ have at least 4 vertices. Moreover, the vertices of $Q$ are located on the boundary of $P$ (because $\bar{\alpha} = 1$) on at least two different sides of $P$ (three vertices cannot be on the same side). It can be shown by inspection that the optimal solution of this NPP instance must have at least four vertices, hence $\rank_+(\mathcal{P}(M)) > 3$, a contradiction. 


Next, we show that $f_4(t) \leq t+1$. Assume there exists $t$ such that $f_4(t) > t+1$. By continuity of $f_4$ with respect to the vertices of  $Q=\conv(\theta(\mathcal{P}^{\bar{\alpha}}(M)))$, there exists $\epsilon > 0$ sufficiently small such that $\bar{\alpha} + \epsilon < 1$ and such that the function $f'_4$ for the NPP instance with inner polygon $Q' = \conv(\theta(\mathcal{P}^{\bar{\alpha}+\epsilon}(M)))$ and the same outer polygon $P$ satisfies $f_4(t) > t+1$ hence $\rank_+(\mathcal{P}^{\bar{\alpha}+\epsilon}(M)) \leq 3$, a contradiction. 

In Appendix~\ref{Appa}, we prove that $f_k$ is piecewise constant/strictly convex for any $k$, i.e., $f_k$ is made up of pieces which are either constant or strictly convex, with at most $m+n$ break points corresponding to different solutions to the NPP. 
Therefore, because $f_4$ is continuous and smaller than $t+1$, it can intersect the line $t+1$ only at the break points. Since there are at most $m+n$ such points corresponding to different NPP solutions, the number of solutions of NPP($\mathcal{P}^{\bar{\alpha}}(M)$) with three vertices is smaller than $m+n$. (Notice that the bound is tight since it is achieved by the nested squares example with 8 solutions.) 
\end{proof}

\begin{remark}
If $\conv(\theta(\mathcal{P}(M)))$ has three vertices, they define a feasible solution for the corresponding NPP problem (i.e., $\mathcal{P}(M)$ is separable, see Theorem~\ref{sepgeo}).  However, the number of solutions might be not be finite in that case. Here is an example 
\[
M = \left( \begin{array}{cccc}    0   & 0.5&   0.25   &      0\\
    1 &   0.5  &  0.75 &   1\\
    1&        0  &  0.1 &   0.5\\
         0 &   1  &  0.9 &   0.5 \\\end{array} \right) 
         \quad \text{ and } \quad  
         \mathcal{P}(M) = \left( \begin{array}{cccc}  
              0   & 0.5    &     0    &     0\\
    1 &  0.5   & 0.3  &  0.5\\
    1  &       0   & 0  &       0\\
         0  &  1   & 0.3  &  0.5 \\ \end{array} \right), 
\]
whose corresponding NPP problems are represented on Figure~\ref{ceth8}: the NPP of $\mathcal{P}(M)$ does not have a finite number of solutions. 
\begin{figure}[ht!]
\begin{center}
\includegraphics[width=5.5cm]{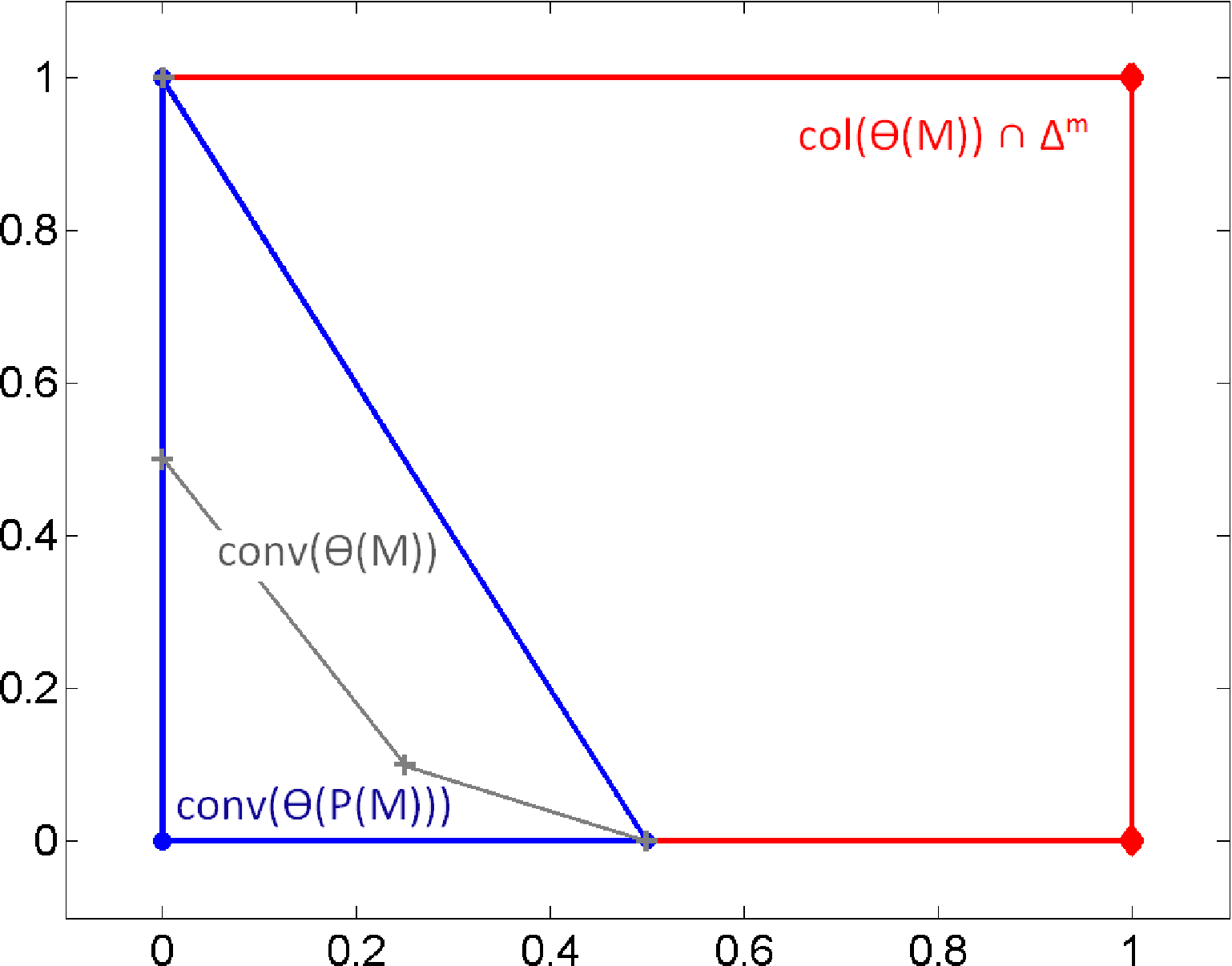}
\caption{Counter-example for Theorem~\ref{finite3} when $\mathcal{P}(M)$ has three vertices.} 
\label{ceth8}
\end{center}
\end{figure}
\end{remark}

The fact that the NPP of the matrix $\mathcal{P}^{\bar{\alpha}}(M)$ can have several different  solutions  is untypical and, we believe, could be due to the symmetry of the problem (as in Example~\ref{2sqex}). 
 We conjecture that, in general, the solution to NPP($\mathcal{P}^{\bar{\alpha}}(M))$ is unique. In particular, we observed on randomly generated matrices that it was, see, e.g., Example~\ref{sepex}. In fact, as the function $f_k(.)$ defined in Theorem~\ref{finite3} depends continuously on the inner and outer polytopes $Q$ and $P$, if these polytopes are generated randomly, there is no reason for the values of the function $f_k(.)$ at the 
 break points to be located on the same line as on Figure~\ref{2sqfail}. 
 
 We also conjecture that Theorem~\ref{finite3} holds true for any rank: 
\begin{conjecture} 
Let $M$ be such that $\rank(M) = \rank_+(M) = k$ and $\conv(\theta(\mathcal{P}(M)))$ has at least $(k+1)$ vertices, and $\bar{\alpha}$ be defined as in Equation \eqref{supr}, then the number of solutions of NPP($\mathcal{P}^{\bar{\alpha}}(M)$) is finite. 
\end{conjecture}

Unfortunately, the geometric construction of Aggarwal et al.\@~\cite{ABOS89} cannot be generalized to three dimensions (or higher). 
To prove the conjecture, we would need to show that 
\begin{itemize}
\item Any solution of NPP($\mathcal{P}^{\bar{\alpha}}(M)$) is isolated. Intuitively, the preprocessing $\mathcal{P}^{\bar{\alpha}}(M)$ of $M$ grows the inner polytope $Q$ as long as the corresponding NPP instance has a solution with $\rank_+(M)$ vertices. If a solution was not isolated, it could be moved around while remaining feasible, which indicates that we could grow the inner polytope $Q$ hence increase $\bar{\alpha}$. 
\item The number of isolated solutions is finite. We conjecture that the solutions can be characterized in terms of the faces of $P$ and $Q$, which are finite (depending on $m$ and $n$).    \\ 
\end{itemize}



\begin{remark}
Of course computing $\bar{\alpha}$ is non-trivial. However, for matrices of small rank, this could be done effectively. In fact, checking whether the nonnegative rank of an $m$-by-$n$ is equal to $\rank(M)$ can be done in polynomial time in $m$ and $n$ provided that the rank is fixed \cite{AGKM11}. In particular, the algorithm of Aggarwal et al.\@~\cite{ABOS89} does it in $\mathcal{O}((m+n)\log(\min(m,n)))$ operations for rank-three matrices \cite{GG10b}. Hence, one could for example use a bisection method to find a good lower bound $\beta \lesssim \bar{\alpha}$ and use the corresponding matrix NPP($\mathcal{P}^{\beta}(M)$) to have a more well-posed NMF problem whose solutions will be solutions of the original one. 
\end{remark}


\section{Preprocessing in Practice}  \label{comp2}

In this section, we address three important practical considerations of the preprocessing.

\subsection{Computational Complexity of Solving \eqref{tigrel}} \label{compcost}

It is rather straightforward to check that problem \eqref{tigrel} can be decoupled into $n$ independent CLLS's, each corresponding to a different column of $M$; for example, for the $i$th column of $M$, we have  
\begin{equation} \label{tigrel2} 
\min_{b \in \mathbb{R}^{n}_+} \quad || M_{:i} - Mb||_2^2 
\quad \text{ such that } \quad  M_{:i} \geq M b , \; b_{i} = 0. 
\end{equation}
We then have $n$ CLLS's with $n$ variables (actually $n-1$ since variable $b_{i} = 0$ can be removed) and $m+n$ constraints. Using interior point methods, the computational complexity for solving \eqref{tigrel2} is of the order of  $\mathcal{O}(n^{3.5})$; hence the total computational cost is of the order  $\mathcal{O}(n^{4.5})$. 

Figure~\ref{CT} shows the computational time needed for solving \eqref{tigrel} with respect to $m$ for $n$ fixed and vice versa, for randomly generated matrices (using the \emph{rand(.)} function of \matlab) on a laptop 3GHz Intel$^{\textrm{\textregistered}}$ CORE i7-2630QM CPU @2GHz 8Go RAM running \matlab \, R2011b using the function \emph{lsqlin(.)} of \matlab.  
\begin{figure*}[ht!]
\begin{center}
\includegraphics[width=16cm]{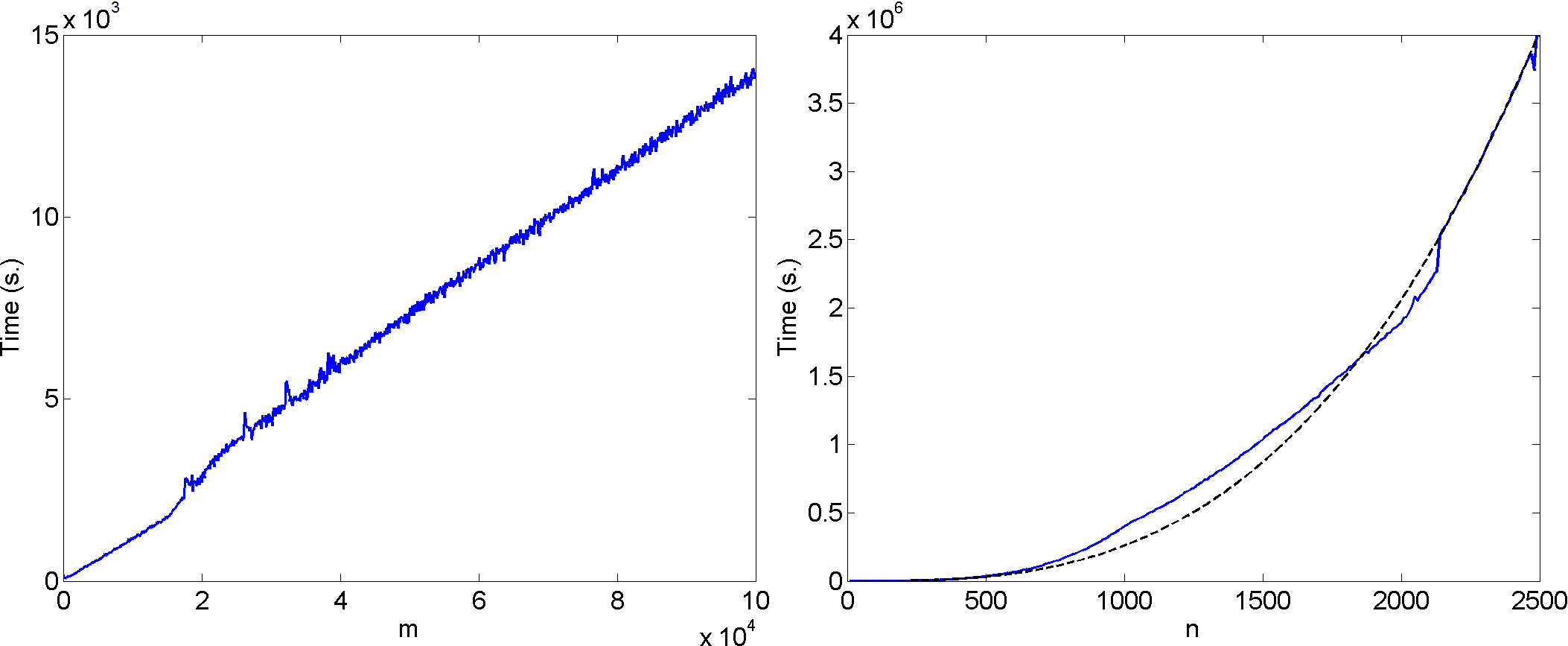}
\caption{Computational time for solving \eqref{tigrel}. On the left, $m$-by-100 randomly generated matrices; on the right, 1000-by-$n$ randomly generated matrices (plain) and the polynomial $2.6 \, 10^{-4} n^{3}$ (dashed).} 
\label{CT}
\end{center}
\end{figure*} 
The computational time is  linear in $m$ while being of the order of $n^{3}$ in $n$,  smaller than the expected $\mathcal{O}(n^{4.5})$. 
Therefore, in practice, the dimension $m$ can be rather large while, on a standard machine, $n$ cannot be much larger than 1000. Using parallel architecture would allow to solve larger scale problems (see also Section~\ref{concl}).






\subsection{Normalization of the Columns of $\mathcal{P}(M)$} \label{scaling}

Since the aim eventually is  to provide a good approximate NMF  to the original data matrix $M$, we observed that normalizing the columns of the preprocessed matrix $\mathcal{P}(M)$ to match the norm of the corresponding columns of $M$ gives better results. That is, we replace $\mathcal{P}(M)$ with $D \mathcal{P}(M)$ where 
\[
D_{ii} = \frac{||M_{:i}||_2}{||\mathcal{P}(M)_{:i}||_2} \text{ for all } i, \quad \text{ and } \quad D_{ij} = 0 \text{ for all } i \neq j.
\]
This scaling does not change the nice properties of the preprocessing since $D$ is a monomial matrix, hence $QD$ still is an inverse-positive matrix. This scaling degree of freedom is related to the fact that we fixed the diagonal entries of $Q$ to one, see Section~\ref{cllssec}. 

The reason for this choice is that NMF algorithms are sensitive to the norm of the columns of $M$. In fact, when using the Frobenius norm, we have that the following two problems are equivalent 
\[
\min_{U \geq 0,V \geq 0} \, 
||M - UV||_F^2 \quad 
\equiv \quad 
\min_{X \geq 0,Y \geq 0} \, \sum_{i=1}^n ||M_{:i}||_2^2 \left\| \frac{M_{:i}}{||M_{:i}||_2}  - XY_{:i}\right\|_2^2. 
\]
Therefore, to give each column of $\mathcal{P}(M)$ the same importance in the objective function as in the original NMF problem, it makes sense to use the scaling above. 
This is particularly critical if there are outliers in the dataset: the outliers do not look similar to the other columns of $M$ hence their preprocessing will not reduce much their $\ell_2$-norm (because they are further away from the convex cone generated by the other columns of $M$). Therefore, their relative importance in the objective function will increase in the NMF problem corresponding to $\mathcal{P}(M)$, which is not desirable.



\subsection{Dealing with Noisy Input Matrices and/or Obtaining Sparser Preprocessing} \label{noisy}

Our technique will typically be useless when the input matrix is noisy and sparse. For example, we have 
\[
M =  \left( \begin{array}{cc}
0   &  0   \\
1   &  0   \\
1   &  1  \\
  \end{array}\right),  
  \mathcal{P}(M) = \left( \begin{array}{cc}
  0   &  0   \\
1   &  0   \\
0   &  1  \\
  \end{array}\right)
\quad \text{ while } \quad 
M_n =  \left( \begin{array}{cc}
0   &  \delta   \\
1   &  0   \\
1   &  1  \\
  \end{array}\right) = 
  \mathcal{P}(M_n), 
\]
for any $\delta > 0$. This shows that the preprocessing is very sensitive to small positive entries of $M$. In order to deal with such noisy and sparse matrices, 
we propose to relax the nonnegativity constraint $M \geq MQ$ in \eqref{tigrel}, and solve instead 
\begin{equation} \label{tigreleps}
\min_{B \in \mathbb{R}^{n \times n}_+} 
\quad 
\sum_{i=1}^n \; \Big\| M_{:i} - \sum_{k \neq i} M_{:k} B_{ki} \Big\|_2^2 
\quad \text{ such that } \quad  M_{:i} + \epsilon ||M_{:i}||_{\infty} e  \geq \sum_{k \neq i} M_{:k} B_{ki}, \; \forall i, 
\end{equation} 
where $0 < \epsilon \ll 1$ and $e$ is the vector of all ones of appropriate dimension.  
We will denote the corresponding preprocessing $\mathcal{P}_{\epsilon}(M) = M (I-B^*_{\epsilon})$ where $B^*_{\epsilon}$ is an optimal solution of \eqref{tigreleps}. 
For the example above with $\delta = \epsilon = 10^{-2}$, we obtain
\[
\mathcal{P}_{\epsilon}(M_n) =  \left( \begin{array}{cc}
-10^{-2}  &  10^{-2}   \\
1    &  -10^{-2}    \\
 10^{-4}  &  0.99  \\
  \end{array}\right). 
\]

In practice, this technique also allows to obtain preprocessed matrices with more entries close (or smaller) than zero.  
When choosing the parameter $\epsilon$, it is very important to check whether $\rho(B^*_{\epsilon}) < 1$ so that the rank of $\mathcal{P}_{\epsilon}(M)$ is equal to the rank of $M$ and no information is lost (i.e., we can recover the original matrix $M = \mathcal{P}_{\epsilon}(M) (I-B^*_{\epsilon})^{-1}$ given $\mathcal{P}_{\epsilon}(M)$ and $B^*_{\epsilon}$). 


\section{Application to Image Processing} \label{appl}

In this section, we apply the preprocessing technique to several image datasets. 
By construction, the preprocessing procedure will remove from each image a linear combination of the other images. As we will see, this will  highlight certain localized parts of these images, essentially because the preprocessed matrices are sparser than the original ones. 
We will then show that combining the preprocessing with standard NMF algorithms naturally leads to better part-based decompositions, because sparser matrices lead to sparser NMF solutions, see Section~\ref{nonu}.


A direct comparison between NMF applied on the original matrix and NMF applied on the  preprocessed matrix is not very informative in itself: while the former will feature a lower approximation error, the latter will provide a sparser part-based representation. This does not really tell us whether the improvements in the part-based representation and sparsity are worth the increase in approximation error. For that reason, we chose to compare them with a standard sparse NMF technique, described below, in order to better assess whether the increase in sparsity achieved is worth the loss in reconstruction accuracy. Hence, we compare the following three different approaches: 
\begin{itemize}

\item \textbf{Nonnegative matrix factorization (NMF)}.  It solves the original NMF problem from Equation~\eqref{nmf} using the accelerated HALS algorithm (A-HALS) from  \cite{GG11} (with parameters $\alpha = 0.5$ and $\epsilon = 0.1$ as suggested in \cite{GG11}), which is a block coordinate descent method.  
 
\item \textbf{Preprocessed NMF for different values of $\epsilon$}. It first computes the preprocessed matrix $\mathcal{P}_{\epsilon}(M)$ (cf.\@ Section~\ref{noisy}), then solves the NMF problem for the rescaled  preprocessed matrix $\mathcal{P}_{\epsilon}(M) D \approx U V'$ (cf.\@ Section~\ref{scaling}) using A-HALS and finally returns $(U,V)$ where $V = \argmin_{X \geq 0} ||M-UX||_F^2$. This approach will be denoted Pre-NMF($\epsilon$). (We will also indicate in brackets the error obtained when using $V = V'Q^{-1}$, which will be, by construction, always higher.) 
Notice that the preprocessed matrix may contain negative entries (when $\epsilon > 0$) which is handled by A-HALS. We do not set these entries to zero for two important reasons: (i) we want to preserve the column space of $M$, 
(ii) the negative entries of $M$ lead to sparser NMF solutions. In particular, it was shown that if an entry of $M$, say at position $(i,j)$, is smaller than $-||\max(0,M)||_F$ then $(UV)_{ij} = 0$ for any optimal solution of NMF \cite{GG08}.

\item \textbf{Sparse NMF}. The most standard technique to obtain sparse solutions for NMF problems is to use a sparsity-inducing penalty term in the objective function. In particular, it is well-known that adding an $l_1$-norm penalty term induces sparser solutions (see, e.g., \cite{KP07}), and we therefore solve the following problem:
\begin{equation}
\min_{U, V \geq 0} ||M-UV||_F^2 + \sum_{i=1}^r \mu_i ||U_{:i}||_1, \quad ||U_{:i}||_{\infty} = 1 \; \forall i,  
\tag{sNMF} \label{sNMF}
\end{equation}
where $||x||_1 = \sum_{i} |x_{i}|$, $||x||_{\infty} = \max_{i} |x_{i}|$ and $\mu_i$ are positive parameters  controlling the sparsity of the columns of $U$. In order to solve sNMF, we also use A-HALS which can easily be adapted to handle this situation. 
The $\ell_{\infty}$-norm constraints is not restrictive because of the degree of freedom in the scaling of the columns of $U$ and the corresponding rows of $V$, while it prevents  matrix $U$ to converge to zero.  
The theoretical motivation is that the $l_1$-norm is the convex envelope of the $l_0$-norm (i.e., the largest convex function smaller than the $l_0$-norm) in the $\ell_{\infty}$-ball, see \cite{RFP10} and the references therein.

In order to compare sparse NMF with Pre-NMF($\epsilon$), the parameters $\mu_i$ $1 \leq i \leq r$ are tuned in order to match the sparsity obtained by Pre-NMF($\epsilon$). The corresponding approach will be denoted sNMF($\epsilon$). 
\end{itemize}

For each approach, we will keep the best solution obtained among the same ten random initializations (using the \emph{rand(.)} function of \matlab) and each run was allowed 1000 (outer) iterations of the A-HALS algorithm.  
We will use the relative error 
\[
\frac{||M-UV||_F}{||M||_F}
\]
to asses the quality of an approximation. We will also display the error obtain by the truncated singular value decomposition (SVD) for the same factorization rank to serve as a comparison. 
For the sparsity, we use the proportion of non-zero entries\footnote{The negative entries of the preprocessed matrix $\mathcal{P}_{\epsilon}(M)$ for $\epsilon > 0$ will be counted as zeros.} 
\[
s(U) = \; \frac{\# \text{zeros}(U)}{mr} \; \in \; [0,1], \text{ for } U \in \mathbb{R}^{m \times r}. 
\]

Because the solution computed with Pre-NMF does not directly aim at minimizing the error $||M-UV||_F^2$, it is not completely fair to use this measure for comparison. 
In fact, it would be better to compare the quality of the sparsity patterns obtained by the different techniques. For this reason, we use the same post-processing procedure as described in \cite{GG09} which benefits all algorithms: once a solution is computed by one of the algorithms, the zero entries of $U$ are fixed and we minimize $\min_{U\geq 0 ,V \geq 0} ||M-UV||_F^2$ on the remaining (nonzero) entries (again, A-HALS can easily be adapted to handle this situation and we perform 100 additional steps on each solution), and report the new relative approximation error as ``Improved''. 

The code is available at \url{https://sites.google.com/site/nicolasgillis}. \\


\subsection{CBCL Dataset}


The CBCL face dataset\footnote{Available at \url{http://cbcl.mit.edu/software-datasets/FaceData2.html}.} 
is made of 2429 gray-level images of faces with $19 \times 19$ pixels (black is one and white is zero). We look for an approximation of rank $r=49$ as in \cite{LS99}. 
Because of the large number of images in the dataset, the preprocessing is rather slow. In fact, we have seen in Section~\ref{compcost} that it is in $\mathcal{O}(n^{4.5})$ where $n$ is the number of images in the dataset (it would take about one week on a laptop). Therefore, we only keep every third image for a total of 810 images, which takes less than three hours for the preprocessing; about 10-15 seconds per image\footnote{The \matlab \, function \emph{lsqlin} for solving CLLS problems is much slower than \emph{quadprog} with interior point (which is much faster than \emph{quadprog} with active set).}.

Table~\ref{cbclM} reports the sparsity and the value of $\rho(B^*_{\epsilon})$ for the preprocessed matrices with different values of the parameter $\epsilon$. As explained in Section~\ref{noisy}, the sparsity of $\mathcal{P}_{\epsilon}(M)$ increases with $\epsilon$, and $\epsilon$ was chosen to make sure that $\rho(B^*_{\epsilon}) < 1$ implying $\rank(\mathcal{P}_{\epsilon}(M)) = \rank(M)$. 
\begin{table}[ht!]
\begin{center}
\caption{CBCL dataset: sparsity of the preprocessed matrices  $\mathcal{P}_{\epsilon}(M) = MQ$ and corresponding spectral radius of $B^*_{\epsilon} = I-Q$.} 
\begin{tabular}{|c||c|c|c|c|}
\hline
\textit{}	& $M$   &   $\mathcal{P}(M)$ &  $\mathcal{P}_{0.05}(M)$ & $\mathcal{P}_{0.1}(M)$  \\ 
\hline
s(.)			  &	 0   &   	0.001 & 	 20.92	& 38.03 \\ 
$\rho(B^*_{\epsilon})$   &	 0   & 0.71 & 	 0.83	& 0.90 \\
\hline  
\end{tabular}
\label{cbclM}
\end{center}
\end{table}

Figure~\ref{cbcl} displays a sample of images of the CBCL dataset along with the corresponding preprocessed images for different values of $\epsilon$. 
\begin{figure}[ht] 
\centering
\includegraphics[width=\textwidth]{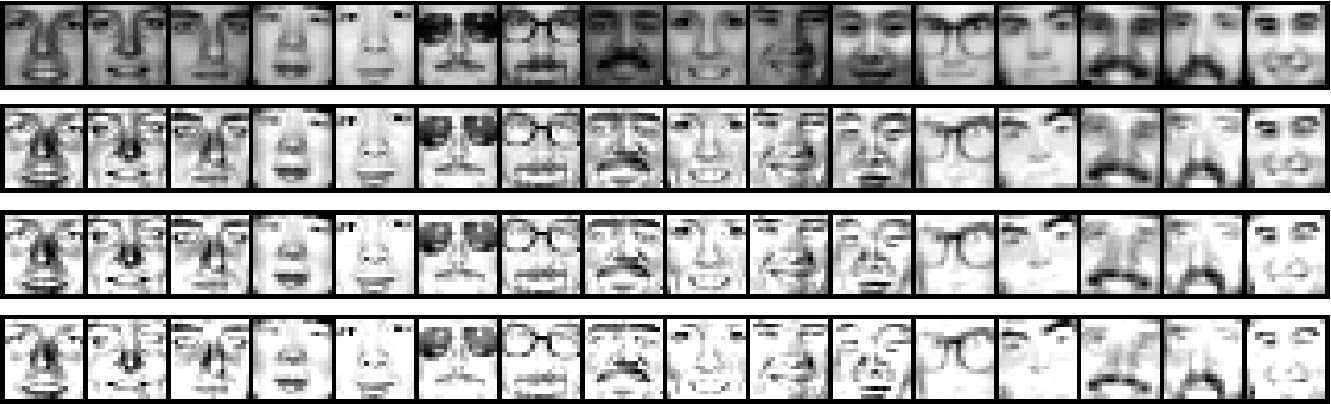} 
\caption{From top to bottom: CBCL sample images, corresponding preprocessed images for $\epsilon=0$, $\epsilon=0.05$,  and $\epsilon=0.1$. } 
\label{cbcl} 
\end{figure}

We observe that the preprocessing is able to highlight some parts of the images: the eyes (faces 5 and 9), the eyebrows (faces 3, 4, 8, 10, 11, 13 and 16), the  mustache (faces 14 and 15), the glasses (faces 6, 7 and 12), the nose (faces 1 to 4) or the mouth (faces 1 to 5). 
Recall that the preprocessing removes from each image of the original dataset a linear combination of other images. Therefore, the parts of the images which are significantly different from the other images are better preserved, hence highlighted. \\

We now compare the three approaches described in the introduction of this section. Table~\ref{cbclT} gives the numerical results and shows that Pre-NMF performs competitively with sNMF in all cases (similar relative error for similar sparsity levels). 
\begin{table*}[ht!]
\begin{center}
\caption{Comparison of the relative approximation error and sparsity for the CBCL image dataset.} 
\begin{tabular}{|c||c|c||c|c|}
\hline
\textit{}	& Plain   &   Improved &  $s(U)$ & $s(V)$ \\ 
\hline
SVD			  &	 7.28   &   	7.28 & 	 0	& 0  \\ \hline	
NMF			  &	 7.97   &   	7.96 & 	 53.27	& 11.36 \\ \hline	
Pre-NMF(0) & \hspace{1.1cm} 9.28 (9.76)$^{\dagger}$    &	   8.37 &   76.78	& 4.42 \\ 
sNMF(0)      &  8.34  &  	8.20 & 	 77.62	& 5.19  \\ \hline 
Pre-NMF(0.05) & \hspace{1cm} 11.12 (12.66) &	   9.15 &   90.14	& 2.16 \\ 
sNMF(0.05)  & 9.24  &   	8.90 & 	91.12	& 2.22 \\  \hline 
Pre-NMF(0.1)  & \hspace{1cm} 13.12  (23.47)  &	  9.88 &   {94.58}	& 1.17 \\
sNMF(0.1)     &  10.30 \hspace{0.05cm}  &   	9.89 & 	 94.77	& 1.14  \\  \hline 
\end{tabular}
\label{cbclT} 
\vspace{-0.2cm}
\begin{flushleft}$^{\dagger}$\footnotesize{
In brackets, it is the error obtained when using $V = V'Q^{-1}$, instead of $V = \argmin_{X\geq0} ||M-UX||_F^2$.}
\end{flushleft}
\end{center}
\end{table*}

Figure~\ref{cbclB} displays the basis elements obtained for NMF, Pre-NMF(0), Pre-NMF(0.1) and sNMF(0.1). 
\begin{figure}[ht] 
\centering
\includegraphics[width=14cm]{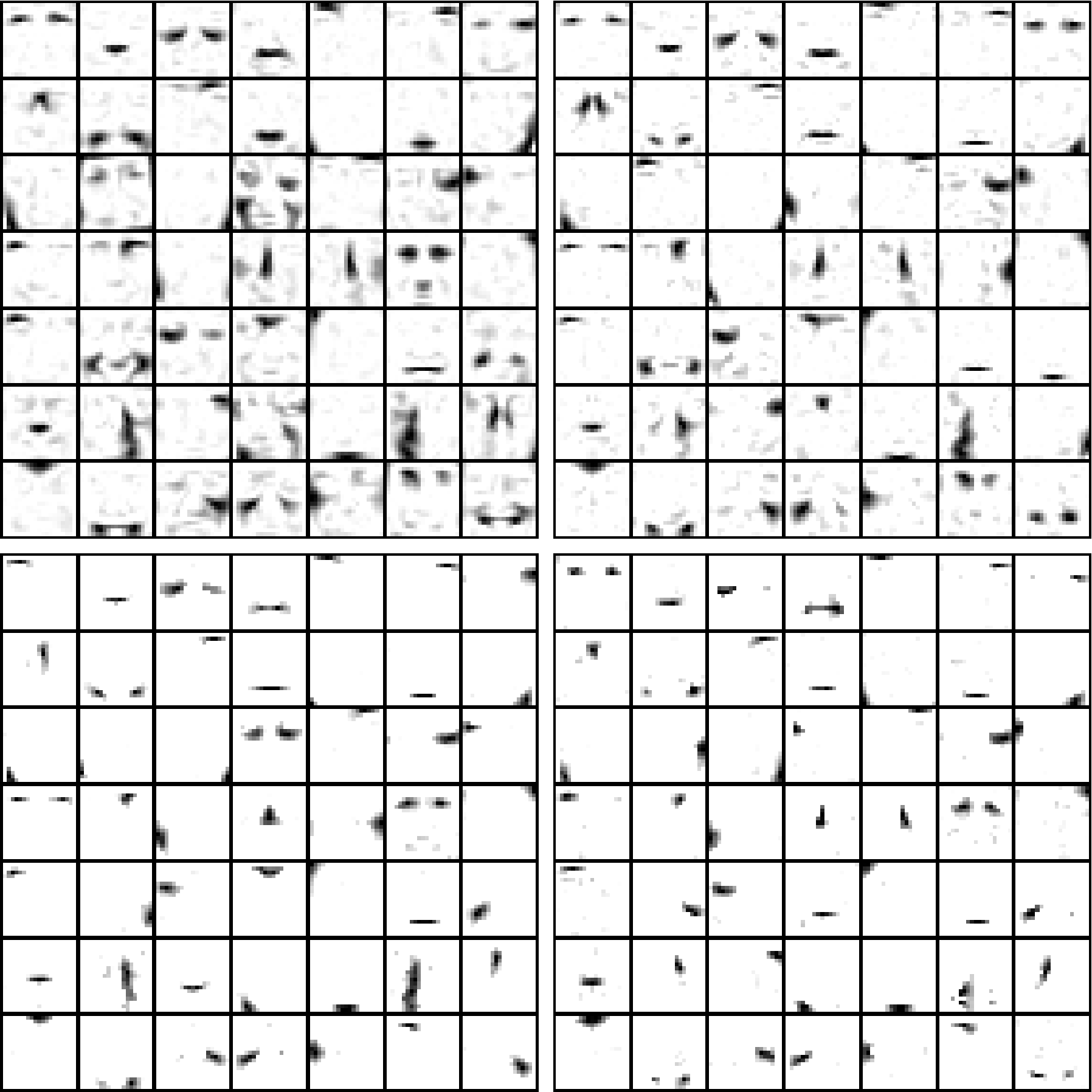} 
\caption{From left to right, top to bottom: basis elements for the CBCL dataset obtained with NMF, Pre-NMF(0), Pre-NMF(0.1) and sNMF(0.1). }
\label{cbclB} 
\end{figure} 
The decomposition by parts obtained by Pre-NMF(0.1) is comparable to sNMF(0.1),  reinforcing the observation above (cf.\@ Table~\ref{cbclT}) that Pre-NMF performs competitively with sNMF. \\

Our technique has the advantage that only one parameter has to be chosen (namely $\epsilon$) and that sparse solutions are naturally obtained. In fact, the user does not need to know in advance the desired sparsity level: one just has to try different values of $\epsilon \in [0,1]$ (making sure $\rho(B^*_{\epsilon}) < 1$) and a sparse factor $U$ will automatically be generated (no parameters have to be tuned in the course of the optimization process). 
Moreover, Pre-NMF proves to be less sensitive to initialization than sNMF: we rerun both algorithms for $\epsilon = 0.1$ with 100 different initializations (using exactly the same settings as above) and observe the following: 
\begin{itemize}

\item Among the hundred solutions generated by sNMF(0.01), three did not achieve the required sparsity (being lower than 0.85, while all others were around 0.95 as imposed). In particular, the variance of the sparsity of the factor $U$ for PreNMF(0.01) is $8.69 \, 10^{-7}$ while it is much higher $1.31 \, 10^{-3}$ for sNMF(0.01). (Note that after removing the three outliers, the variance of sNMF(0.01) is still higher being $3.23 \, 10^{-6}$.) 

\item The average of the relative error of Pre-NMF(0.01) is 9.94, slightly lower than sNMF(0.01) with 9.96.  

\item The variance of the relative  error of Pre-NMF(0.01) is $5.97 \, 10^{-3}$, lower than sNMF(0.01) with $2.36 \, 10^{-2}$.  

\end{itemize}


\begin{remark}
We have also tested other sparse NMF techniques and they could not match the results obtained by sNMF, especially for high sparsity requirement. In particular, we tested the following standard formulation using only one penalty parameter (see, e.g., \cite{KP07}) 
\[
\min_{U, V \geq 0} ||M-UV||_F^2 + \mu \sum_i ||U_{:i}||_1, 
\] 
and the algorithm of Hoyer \cite{Hoy}. 
\end{remark}



\subsection{Hubble Telescope Hyperspectral Image} \label{Hubblesec}


The Hubble dataset consists of 100 spectral images of the Hubble telescope, $128 \times 128$ pixels each \cite{PPP06}. It is composed of eight materials\footnote{These are true Hubble satellite material spectral signatures provided by the NASA Johnson Space Center.}; see the fourth row on  Figure~\ref{BasisH}. 
The preprocessing took about one minute (about 0.5 second per image)\footnote{The \matlab \, function \emph{lsqlin} for solving CLLS problems was again much slower (about ten times) than \emph{quadprog} with active set or with interior point which were comparable in this case.}.
Figure~\ref{SampleH} displays a sample of images of the simulated Hubble database along with the corresponding preprocessed images. 
\begin{figure}[ht]
\begin{center}
\includegraphics[width=\textwidth]{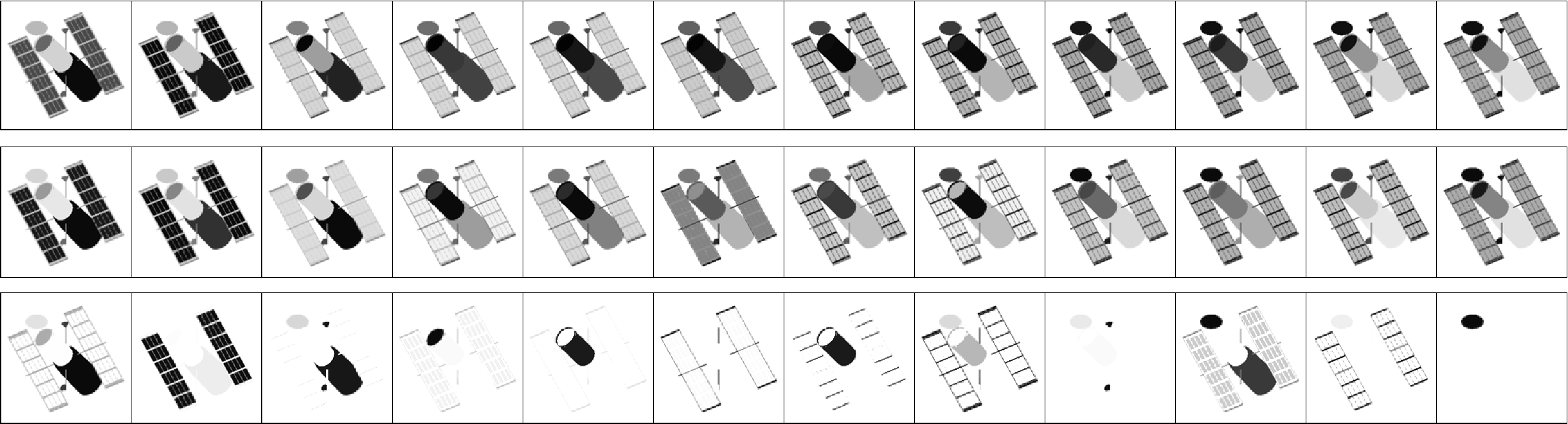}
\caption{From top to bottom: Sample of Hubble images, corresponding preprocessed images for $\epsilon = 0$ and $\epsilon = 0.01$.} 
\label{SampleH} 
\end{center}
\end{figure}
The preprocessing for $\epsilon = 0.01$ highlights extremely well the constitutive parts of the Hubble telescope: it is in fact able to extract some materials individually. 
Table~\ref{HubbleM} gives the sparsity and the value of $\rho$ for the different preprocessed matrices. 
\begin{table}[ht!]
\begin{center}
\caption{Hubble dataset: sparsity of the preprocessed matrices  $\mathcal{P}_{\epsilon}(M) = MD$ and corresponding spectral radius of $B^*_{\epsilon} = I-D$.}  
\begin{tabular}{|c||c|c|c|}
\hline
\textit{}	& $M$   &   $\mathcal{P}(M)$ &   $\mathcal{P}_{0.01}(M)$\\
\hline
s(.)			  &	57 & 57 & 80  	  	   \\
$\rho(B^*_{\epsilon})$		& 0  & 0.9808 & 0.9979 	    \\
\hline 
\end{tabular}
\label{HubbleM}
\end{center}
\end{table}

Table~\ref{HubbleT} reports the numerical results. Although sNMF(0.01) identifies a solution with slightly lower reconstruction error than Pre-NMF(0.01) (2.90 vs.\@ 2.93), it is not able to identify the constitutive materials properly while Pre-NMF(0.01) \emph{perfectly separates all eight constitutive materials}. 
It is also important to point out that the solutions generated by Pre-NMF(0.01) with different initializations correspond in most cases\footnote{We used 100 random initializations and obtained 61 times the optimal decomposition (in the other cases, it is always able to detect at least six of the eight materials).} to this optimal decomposition  while the solutions generated by sNMF are typically very different (and with very different objective function values). This indicates that the NMF problem corresponding to the preprocessed matrix is more well posed\footnote{Of course, in general, even if an NMF formulation has a unique global minimum (up to permutation and scaling), it will still have many local minima. Therefore, even in that situation, solutions generated with standard nonlinear optimization algorithms might still be rather different for different initializations.}. 
\begin{table}[ht!]
\begin{center}
\caption{Comparison of the relative approximation error and sparsity for the Hubble dataset.}
\begin{tabular}{|c||c|c||c|c|}
\hline
\textit{}	& Plain   &   Improved &  $s(U)$ & $s(V)$ \\ 
\hline
SVD			  &	 \hspace{1cm} 0.01 \hspace{1cm} &   	0.01  & 	 58	& 0  \\
\hline	
NMF			  &	 0.06   &  	0.05  & 	 58.02	& 2.25 \\
\hline	
Pre-NMF(0) & \hspace{1cm} 0.08 (0.08)  &	 0.07   &  59.16 	& 0.13 \\
sNMF(0)    &  0.37   & 0.36   & 64.14	 	&  0.63  \\ 
\hline 
Pre-NMF(0.01) & \hspace{1.2cm} 14.08 (75.09)$^{\dagger}$   &	  2.93 &   93.71	& 0 \\
sNMF(0.01)    &  3.39   &   	2.90  & 	 93.94	& 0  \\ 
\hline 
\end{tabular} 
\label{HubbleT}
\vspace{-0.2cm}
\begin{flushleft}$^{\dagger}$\footnotesize{Notice that for $\epsilon = 0.01$, the solution obtained using $V = V'Q^{-1}$ has a very high reconstruction error; the reason being that $Q = (I-B^*_{\epsilon})$ is close to being singular since $\rho(B^*_{\epsilon}) = 0.9979$.} 
\end{flushleft}
\end{center}
\end{table}

The comparison between sNMF(0) and Pre-NMF(0) is also interesting: the basis elements generated by Pre-NMF(0) (see second row of Figure~\ref{BasisH}) identify the constitutive materials much more effectively as six of them are almost perfectly extracted, while sNMF(0) only identifies one (while another is extracted as two separate basis elements). 
\begin{figure}[ht!] 
\begin{center}
\includegraphics[width=14cm]{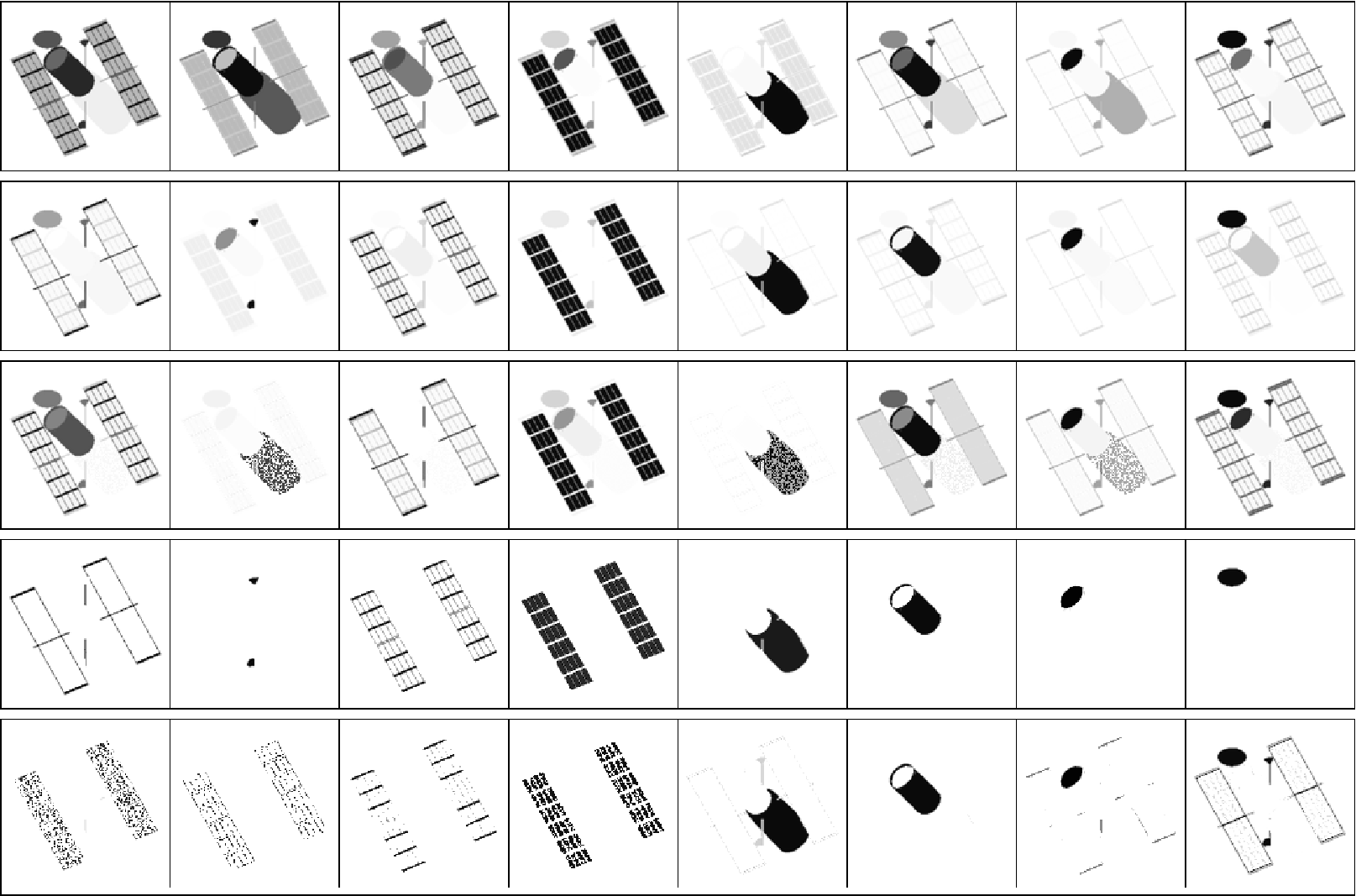}
\caption{From top to bottom: basis elements for the Hubble dataset obtained by NMF, Pre-NMF(0), sNMF(0), Pre-NMF(0.01) and sNMF(0.01).} 
\label{BasisH} 
\end{center}
\end{figure}



\section{Conclusion and Further Research} \label{concl}

In this paper, we introduced a completely new approach to make NMF problems more well posed  and have sparser solutions. It is based on the preprocessing of the  nonnegative data matrix $M$: given $M$, we compute an inverse positive matrix $Q$ such that the preprocessed matrix $\mathcal{P}(M) = MQ$ remains nonnegative and is sparse. The computation of $Q$ relies on the resolution of constrained linear least squares problems (CLLS). We proved that the preprocessing is well-defined, invariant to permutation and scaling of the columns of matrix $M$, and preserves the rank of $M$ (as long as the vertices of $\conv(\theta(M))$ are non repeated). 
 
Because $\mathcal{P}(M)$ is sparser than $M$, the corresponding NMF problem will be more well posed and have sparser solutions. In particular, we were able to show that 
\begin{itemize}
\item Under the separability assumption of Donoho and Stodden \cite{DS03}, the preprocessing is optimal as it identifies the vertices of the convex hull of the columns of $M$. 
\item Since any rank-two matrix satisfies the separability assumption, the preprocessing is optimal for any rank-two matrix. 
\item In the exact rank-three case (i.e., $M = UV$, $\rank(M) = \rank_+(M) = 3$), the preprocessing can be used to make the set of optimal solutions of the NMF problem finite. 
We conjecture that, generically, it makes it unique and that this result holds for higher rank matrices. 
\end{itemize}

We also proposed a more general preprocessing that relaxes the constraint that $\mathcal{P}(M)$ has to be nonnegative, which is able to deal better with noisy and sparse matrices. Moreover, it generates sparser preprocessed matrices hence sparser NMF solutions. We experimentally showed the effectiveness of this strategy on  facial and hyperspectral image datasets. In particular, it performed competitively as a state-of-the-art sparse NMF technique based on $\ell_1$-norm penalty functions. It is robust to high sparsity requirement and no parameters have to be tuned in the course of the optimization process. Only one parameter has to be chosen which will allow the user to generate more or less sparse preprocessed matrices. \\

The main drawback of the technique seems to be its computational cost: $n$ CLLS problems in $n$ variables and $m+n$ constraints have to be solved (where $n$ in the number of columns of $M$) for a total computational cost of the order of $\mathcal{O}(n^{4.5})$ (using \matlab \, on a standard laptop, it limits $n$ to be smaller than 1000 for a few hours of computation). 
It would then be particularly interesting to investigate strategies to speed up the preprocessing. Using faster solvers is one possible approach (probably in detriment of the accuracy), e.g., based on first-order methods\footnote{We are  currently developing an ADM algorithm, along with Ting Kei Pong, which allowed us to preprocess the CBCL dataset in about 10 hours with $10^{-3}$ relative accuracy; the code is available upon request and should be available soon.}. 
 Another possibility would be to use the following heuristic: since the preprocessing removes from each column of $M$ a linear combinations of the other columns, one could use only a subset of $k$ columns of $M$ to be subtracted from the other columns of $M$. This amounts to fixing variables to zero in the CLLS problems and would reduce the computational complexity to $\mathcal{O}(n k^{3.5})$.  This subset of columns could for example be selected such that its convex hull has a large volume, see, e.g., \cite{KCD09} for a possible heuristic; or such that they form the best possible basis for the remaining columns (i.e., use a column subset selection algorithm); see \cite{BMD09} and the references therein.  

Finally, a particularly challenging direction for research would be to design other data  preprocessing techniques for NMF. One approach would be to  characterizing the set of inverse positive matrices better: in this paper, we only worked with the subset of invertible M-matrices.  
For example, the matrix\footnote{We thank Mariya Ishteva for providing us with this example.}  
\[
M = \left( \begin{array}{ccc}
0 & 1 & 1 \\ 
1 & 0 & 1 \\
1 & 1 & 0 \end{array} \right)
\]
would not be modified by our preprocessing (because each column contains a zero entry corresponding to positive ones in all other columns) although its NMF is not unique (cf.\@ Section~\ref{nonu}). In fact, we have 
\[
MQ =  \left( \begin{array}{ccc}
0 & 1 & 1 \\ 
1 & 0 & 1 \\
1 & 1 & 0 \end{array} \right) 
\left( \begin{array}{ccc}
-1 & 1 & 1 \\ 
1 & -1 & 1 \\
1 & 1 & -1 \end{array} \right) = 2 
\left( \begin{array}{ccc}
1 & 0 & 0 \\ 
0 & 1 & 0 \\
0 & 0 & 1 \end{array} \right), 
\] 
where $Q$ is inverse positive with  $Q^{-1} = \frac{1}{2}M$, and the NMF of $MQ$ is unique. 
This examples shows that working with a larger set of inverse positive matrices would allow to obtain sparser preprocessed data matrices, hence more well-posed NMF problems with sparser solutions.  


\section*{Acknowledgment} 

The author acknowledges a discussion with Mariya Ishteva about uniqueness issues of NMF  which motivated the study of inverse-positive matrices in this context.   
The author would also like to thank K.C.\@ Sivakumar and F.-X. Orban de Xivry for helpful discussions on inverse positive matrices and on the problem of finding the closest stable matrix to a given one, respectively, and Stephen Vavasis for carefully reading and commenting a first draft of this manuscript.

\bibliographystyle{siam}
\bibliography{preprocNMF}

\appendix

\section{Proof for Theorem~\ref{finite3}} \label{Appa}

In this section, we prove that the function $f_k$ defined in Theorem~\ref{finite3}  is continuous and made up of pieces which are either constant or strictly convex (which we refer to as piecewise constant/strictly convex). The construction described below is the same as the one proposed by Aggarwal et al.\@~\cite{ABOS89} and we refer the reader to that paper for more details. The novelty of our proof is to use that construction to show that $f_k$ is piecewise constant/strictly convex (it was already shown to be continuous and nondecreasing in \cite{ABOS89}). 

\begin{proof} 
Let $x(t_1)$ be on the boundary of $P$ and define the sequence $x(t_2)$, $\dots$, $x(t_{k+1})$ as in Theorem~\ref{finite3} (clock-wise). 
As shown by Aggarwal et al.\@~\cite{ABOS89}, the function $f_k(t_1) = t_{k+1}$ only depends on 
\begin{enumerate} 
\item The sides of $P$ on which the points $x(t_i)$ $1 \leq i\leq k+1$ lie ; 
\item The intersections of the segments $[x(t_i),x(t_{i+1})]$ $1 \leq i\leq k$ with $Q$ ;  
\end{enumerate}
and, given that these sides and intersections do not change, $f_k$ is continuously differentiable and can be characterized in closed form (see below). These sides and intersections will change when either 
\begin{itemize} 
\item One of the points $x(t_i)$ switches from one side of the boundary of $P$ to another. These points correspond to the vertices of $P$ ($P$ has at most $m$ vertices since it is a polygon defined with $m$ inequalities); or, 

\item One of the intersections of the segments $[x(t_i),x(t_{i+1})]$ $1 \leq i\leq k$ with $Q$ changes. There is a one-to-one correspondence between these points and the sides of $Q$ ($Q$ has at most $n$ vertices hence at most $n$ sides). 
\end{itemize}
These points where the description of $f_k$ changes (and where $f_k$ is not continuously differentiable) are called the contact change points. Turning around the boundary of $P$, we might encounter more than $m+n$ such points. However, two contact change points corresponding to the same change are associated with the same sequence $x(t_i)$ $1 \leq i \leq k+1$ hence the same solution to the NPP. In fact, both sequences must share at least one point (either a vertex of $P$ or the intersections of a line containing a side of $Q$ with the boundary of $P$) which implies, by construction, that they are the same. 
Therefore, there are at most $m+n$ contact change points corresponding to different sequences $x(t_i)$ $1 \leq i \leq k+1$ on the boundary of $P$ \cite{ABOS89}. 

It remains to show that the pieces of $f_k$ between two contact change points are either constant or strictly convex. \\ 


Let us then construct the function $f_k$ between two contact change points. Without loss of generality, we may assume that the perimeter of the outer polygon $P$ is equal to one (otherwise scale the polygons $P$ and $Q$ accordingly), and that the parametrization $x$ of the boundary of $P$ has the following property: the distance traveled when following the boundary between $x(s)$ and $x(t)$ is equal to $|(s - \left\lfloor  s \right\rfloor) - (t - \left\lfloor  t \right\rfloor)|$. In particular, if $0 \leq s \leq t \leq 1$, then the distance traveled between $x(t)$ and $x(s)$ along the boundary of $P$ is $t-s$. We may also assume without loss of generality that $x(0) = (0,0)$ is the vertex on $P$ preceding $x(t_1)$ and that $x(t_1) = (0,t_1)$: this amounts to translating and rotating $P$ and $Q$. We also define (see Figure~\ref{scf} for an illustration) 
\begin{itemize}
\item $q = (q_1,q_2)$,  the tangent point on $Q$ between  $x(t_1)$ and $x(t_2)$. 
\item $\theta$, the angle between the sides of $P$ on which $x(t_1)$ and $x(t_2)$ are. 
\item $p$, the intersection between the sides on which $x(t_1)$ and $x(t_2)$ are (note that $p$ is on the boundary of $P$ if and only if there is one and only one vertex of $P$ between $x(t_1)$ and $x(t_2)$). 
\item $d$, the distance between $x(0)$ and $p$. 
\item $s$, the distance between $p$ and $x(t_2)$. 
\item $a$, the projection of $q$ on the line $[x(0),p]$. 
\item $b$, the projection of $x(t_2)$ on the line $[x(0),p]$. 
\end{itemize}
\begin{figure}[ht!]
\begin{center}
\includegraphics[width=12cm]{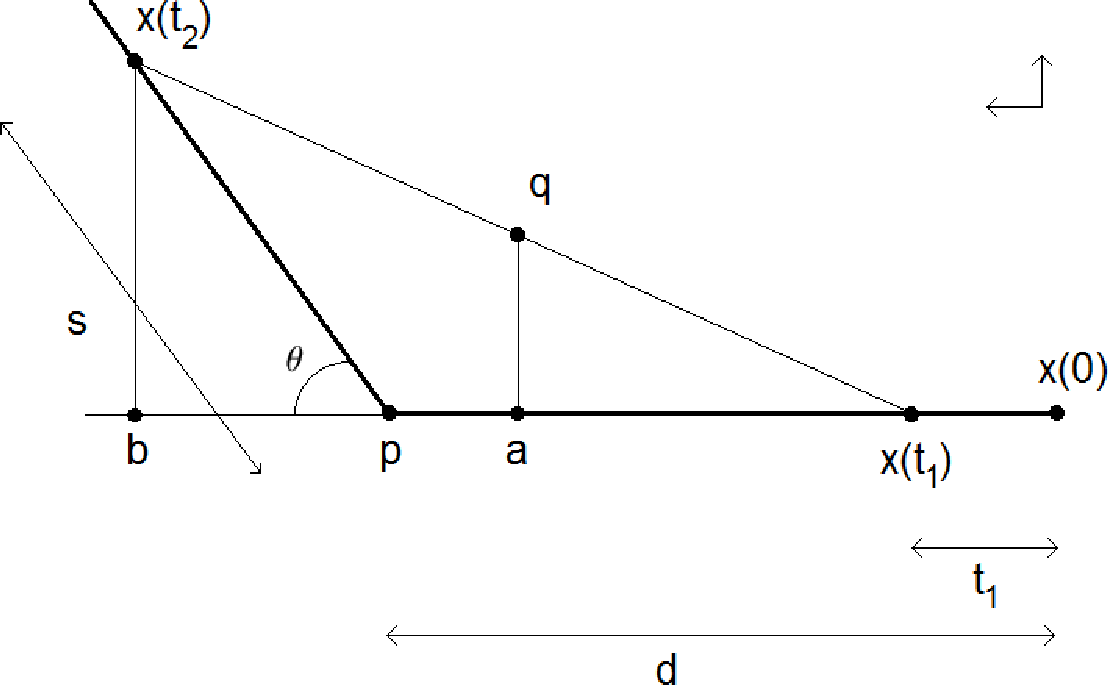}
\caption{Construction of the function $f_1$ between two contact change points (see Aggarwal et al.\@~\cite[Figure~3]{ABOS89} for a similar illustration).} 
\label{scf}
\end{center}
\end{figure}

\noindent \textbf{Case 1:} The point $q$ is on the same side as $x(t_1)$. This implies  that $x(t_2) = p$ for any $t_1 < q_1$ and no other points of the sequence is changed since $x(t_2)$ remains the same. 
Therefore, 
the function $t_{k+1} = f_k(t_1)$ is constant. (Notice that $x(q_1)$ is a contact change point since $x(t_2)$ will switch side when $t_1 = q_1$.)  \\

\noindent \textbf{Case 2:} The point $q$ is on the same side as $x(t_2)$. This implies that $x(t_2) = q$ for any $t_1 < d$. 
Therefore, 
the function $t_{k+1} = f_k(t_1)$ is constant. (Notice that the next contact change point will be the first vertex of $P$ that $x(t_1)$ crosses.)  \\

\noindent \textbf{Case 3:}  The point $q$ is not on the same side as $x(t_1)$ or $x(t_2)$ (i.e., it is in the interior of $P$). Using the similarity between the triangles $\Delta x(t_1) a q$ and $\Delta x(t_1) b x(t_2)$, we have that \cite[Equation~(1)]{ABOS89} 
\[
\frac{q_2}{q_1-t_1} = \frac{s \sin(\theta)}{d-t_1+ s \cos(\theta)}, 
\]
implying
\[
s = \frac{q_2}{\sin(\theta)} \frac{d - t_1}{ q_1 - q_2 \cot(\theta) - t_1} = g_1(t_1).  
\]
Let us show that $g_1(t_1)$ is strictly convex, i.e., $g_1''(t_1) > 0$.  
Since $q$ is not on the same side as $x(t_1)$ or $x(t_2)$, we have $q_2 > 0$ and $0 < \theta < \pi$ implying $\frac{q_2}{\sin(\theta)} > 0$. Hence it suffices to show that $h(t_1) = \frac{d-t_1}{l-t_1}$ is strictly convex, where $l = q_1 - q_2 \cot(\theta)$. 
Since $s > 0$ and $d > t_1$, we must have $l - t_1 > 0$. (Notice that $x(l)$ is a contact change point. In fact, for $t_1 = l$, the segments $[x(t_1),q]$ and $[p,x(t_2)]$ become parallel implying that the intersection of $Q$ with the segment $[x(t_1),x(t_2)]$ will change.) 

We then have 
\[
h'(t_1) = \frac{d-l}{(l-t_1)^2}. 
\] 
Since $h$ is a strictly increasing function of $t_1$ \cite{ABOS89}, $h'(t_1) > 0$ 
hence $d > l$ and  
\[
h''(t_1) = 2 \frac{d-l}{(l-t_1)^3} > 0, 
\]
so that $g_1(t)$ is strictly convex. Finally, we have 
\[
f_1(t_1) = t_2 = c_1 + s = c_1 + g_1(t_1),
\] 
where either
\begin{itemize}
\item $c_1 = 0$ and $g_1$ is a constant (cases 1.\@ and 2.). 
\item $c_1$ is an appropriate constant and $g_1$ is an increasing and strictly convex function (case 3.). 
\end{itemize} 
By construction, the same relationship will apply between $t_2$ and $t_3$ with  
\[
f_2(t_1) = t_3 = c_2 + g_2(s)  = c_2 + g_2(g_1(t_1)), 
\]
where $c_2$ is an appropriate constant and $g_2$ is either constant, or strictly convex and increasing. After $k+1$ steps, we have 
\[
f_k(t_1) = t_{k+1} = c_k + g_k(s)  = c_k + \left(g_k \circ g_{k-1} \circ \dots \circ g_1\right)(t_1),  
\]
where $c_k$ is an appropriate constant and the functions $g_i$ are either constant, or  strictly convex and increasing. 
If one of the functions $g_i$ $1 \leq i \leq k$ is constant, then $f_k$ is constant.  Otherwise the function $f_k(t_1) = c_k + (g_{k-1} \circ \dots \circ g_1)(t_1)$ is strictly convex since it is a constant plus the composition of strictly convex and \emph{increasing} functions. (In fact, the composition of one-dimensional increasing and strictly convex functions is increasing and strictly convex.)  
\end{proof}

\end{document}